\documentclass{article}

\usepackage[main,final]{neurips_2026}

\usepackage[utf8]{inputenc}
\usepackage[T1]{fontenc}
\usepackage{url}
\usepackage{booktabs}
\usepackage{amsfonts}
\usepackage{nicefrac}
\usepackage{microtype}
\usepackage{xcolor}

\usepackage{graphicx}
\usepackage{subcaption}
\usepackage{multirow}
\usepackage{placeins}
\usepackage{wrapfig}
\usepackage{amsmath}
\usepackage{amssymb}
\usepackage{mathtools}
\usepackage{amsthm}

\usepackage[hidelinks]{hyperref}
\usepackage[capitalize,noabbrev]{cleveref}

\usepackage[disable,textsize=tiny]{todonotes}

\newcommand{\rev}[1]{#1}

\theoremstyle{plain}
\newtheorem{theorem}{Theorem}[section]
\newtheorem{proposition}[theorem]{Proposition}
\newtheorem{lemma}[theorem]{Lemma}

\theoremstyle{definition}

\newtheorem{assumption}[theorem]{Assumption}
\theoremstyle{remark}

\title{Latent Generative Solvers for Generalizable Long-Term Physics Simulation}

\author{%
  Zituo Chen \\
  MIT MechE \\
  \texttt{zituo@mit.edu} \\
  \And
  Sili Deng\thanks{Corresponding author: \texttt{silideng@mit.edu}.} \\
  MIT MechE \\
  \texttt{silideng@mit.edu} \\
}

\begin{document}

\maketitle

\begin{abstract}
Reliable physics simulation requires both broad generalization across heterogeneous PDE families and stability under long autoregressive rollouts. Current neural solvers rarely provide both. Deterministic operators accumulate rollout error, while probabilistic solvers typically remain tied to a single PDE family or a short prediction horizon. We address this gap with the \textbf{Latent Generative Solver} (LGS), which couples a Physics VAE (PhyVAE) with a Pyramidal Flow-Forcing Transformer (PFlowFT). PhyVAE compresses states from twelve PDE families into a \rev{shared learned latent representation}, separating dynamics-relevant structure from the ambient state space. PFlowFT then predicts the next latent state through input-noised flow matching, with a conditional stability analysis that bounds rollout error under an explicit effective-contraction assumption. Pretrained on a 2.5\,M-trajectory, 16-system corpus at $128^2$ resolution, LGS obtains the lowest relative $L_2$ error (L2RE) on 15 of 16 systems at both 5- and 10-step rollout. At 20 steps over 11 eligible systems, LGS reduces the L2RE from $56.1\%$ to $\mathbf{30.2\%}$ compared with U-AFNO, while requiring $\mathbf{13}$--$\mathbf{77\times}$ less \rev{training-time} recurrent dynamics-step compute. On a held-out $256^2$ Kolmogorov flow, LGS reduces the 1-step L2RE from $0.398$ to $0.129$ within five finetuning epochs; U-AFNO achieves only $0.653$ to $0.343$ under the same protocol. \rev{These empirical claims are limited to the evaluated 2D regular-grid systems and discrete horizons.}
\end{abstract}

\section{Introduction}

A useful neural PDE solver must generalize across heterogeneous physical systems while remaining stable under long autoregressive rollouts. Existing methods typically satisfy only one of these requirements. Scalable operators~\citep{mccabe2023multiple,cao2024vicon,ye2025pdeformerfoundationmodelonedimensional,zhou2025unisolverpdeconditionaltransformersuniversal} cover many systems, but their deterministic transitions accumulate error once predictions are fed back as inputs~\citep{lippe2023pderefinerachievingaccuratelong,ye2025recurrentneuraloperatorsstable}. Probabilistic alternatives~\citep{huang2024diffusionpdegenerativepdesolvingpartial,wang2025fundiffdiffusionmodelsfunction,feng2025fluidzero} improve rollout stability but typically stay tied to a single PDE family or a short prediction horizon. Closing this gap requires rethinking both \emph{where} dynamics are modeled and \emph{how} the one-step transition is parameterized. We propose a \emph{latent generative} solver: heterogeneous states share a \rev{learned latent representation}, and each one-step transition is realized as a flow-matching generative transition.

\paragraph{Why latent.} \emph{Raw fields are not naturally transferable across systems; learned latent states can be.} PDE states contain discretization-, resolution-, and channel-specific structure that changes across datasets. \rev{PhyVAE maps standardized heterogeneous fields to a common compact representation. We evaluate what this representation retains through reconstruction, downstream rollout, and latent-capacity ablations.}

\paragraph{Why generative.} \emph{One-step accuracy alone does not ensure accurate rollout.} Once predictions drift from the training distribution, subsequent updates may propagate the deviation. LGS trains a generative transition on perturbed latent inputs to improve robustness to such deviations. Its rollout stability is evaluated empirically and analyzed under an explicit effective-contraction assumption.

\begin{figure}[!t]
\centering
\begin{minipage}[c]{0.50\linewidth}
\centering
\includegraphics[width=\linewidth]{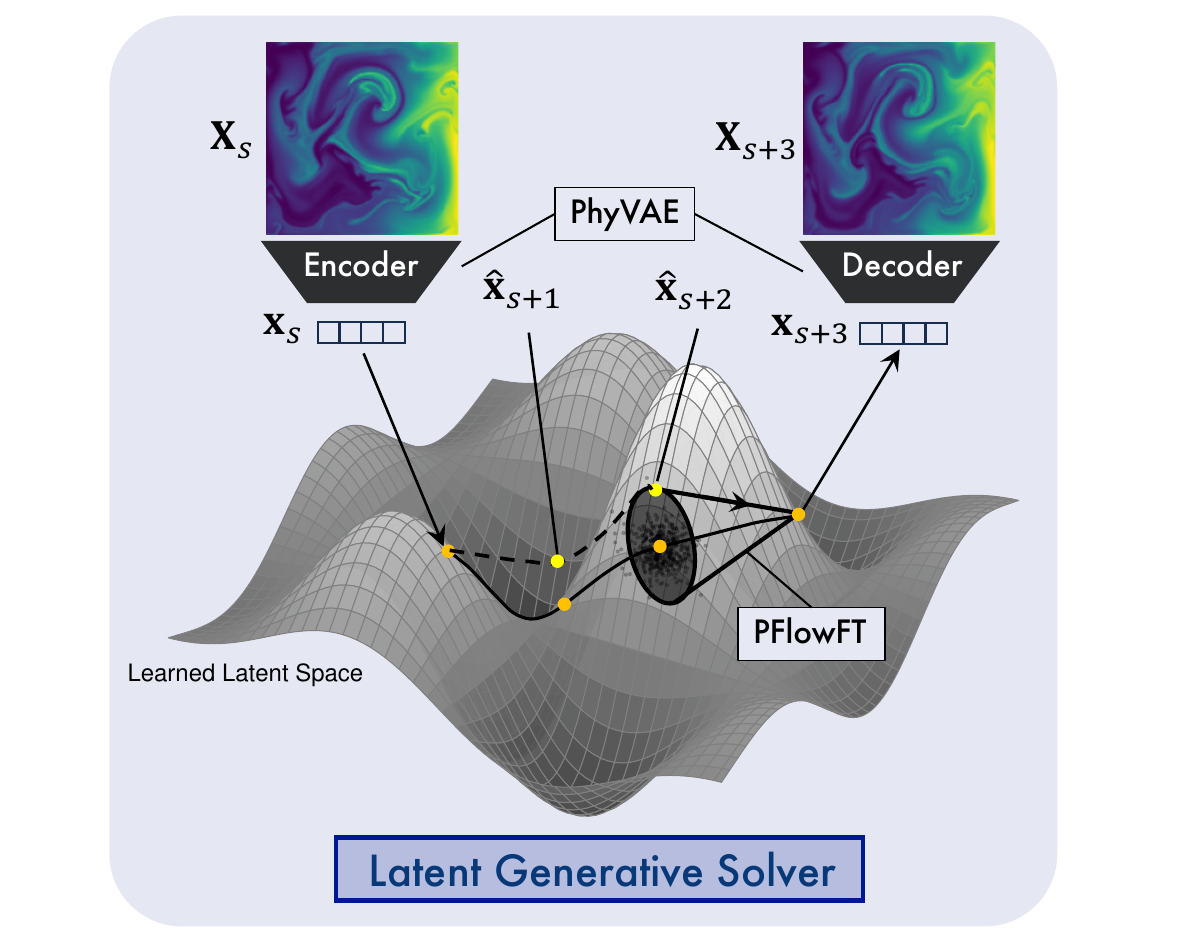}
\end{minipage}\hfill
\begin{minipage}[c]{0.46\linewidth}
\caption{\textbf{Overview of LGS.} PhyVAE encodes a physical field $\mathbf{X}_s$ into a shared learned latent representation $\mathbf{x}_s$. PFlowFT predicts the next latent state with a generative transition conditioned on a fixed four-frame history; the decoder maps the prediction back to a physical field. The gray surface illustrates learned latent support, rather than a universal physical state space.}
\label{fig:overview}
\end{minipage}
\end{figure}

\paragraph{Our framework.} We instantiate these principles in the \emph{Latent Generative Solver} (LGS, Fig.~\ref{fig:overview}). The first component, a Physics VAE (\textbf{PhyVAE}), compresses fields from twelve PDE families into a shared $16{\times}16{\times}16$ \emph{learned latent representation}, separating dynamics-relevant structure from the discretization-specific ambient space. The second component, a Pyramidal Flow-Forcing Transformer (\textbf{PFlowFT}), is the main methodological contribution. PFlowFT generates the next latent state by flow matching, conditioned on a per-trajectory \emph{context} $c$ updated from the model's own predictions. \rev{The fixed four-frame history is spatially pooled through a pyramid to reduce attention cost}, and training uses \emph{input noising} to expose the transition to rollout-like deviations. These mechanisms are coupled. We analyze bounded rollout error under an explicit effective-contraction assumption on the ideal transition (\cref{sec:errcompare_main}), and evaluate the training mechanisms through ablations.

\paragraph{Contributions.}

\begin{itemize}
    \item A \textbf{time-forward generative solver}, the Pyramidal Flow-Forcing Transformer (PFlowFT), built on three principles that target failure modes specific to autoregressive rollout: input-noised correction, flow forcing of the per-trajectory context on model-generated states, and \rev{pyramidal compression of a fixed four-frame history}. PFlowFT extends static latent generative models such as DiT~\citep{peebles2023scalablediffusionmodelstransformers} and SiT~\citep{ma2024sitexploringflowdiffusionbased} to the iterated-rollout setting, with a sufficient-condition stability analysis (\cref{sec:errcompare_main}) consistent with the observed long-horizon behavior.
    \item \textbf{Long-horizon stability across heterogeneous PDEs.} Against matched deterministic and adapted generative baselines, LGS attains the lowest L2RE on \textbf{15 of 16} systems at both 5- and 10-step rollout, $30.2\%$ vs.\ U-AFNO's $56.1\%$ at 20 steps, at $13$--$77\times$ less \rev{training-time} recurrent dynamics-step compute. Scaling U-AFNO and DPOT to $\sim$400\,M parameters does not close the 10-step gap ($23.9\%$ and $19.7\%$ vs.\ LGS $18.3\%$; Appendix~\cref{tab:scaled_baselines}). LGS also adapts to a $256^2$ Kolmogorov flow held out from pretraining.
    \item A \rev{\textbf{standardized preprocessing and evaluation protocol}} consolidating four public PDE collections (FNO-v, PDEArena, PDEBench, The Well) into a fair-comparison protocol over 12 PDE families at $128^2$, with reproducible deterministic and generative baselines at matched active-parameter and predictor-update budgets.
\end{itemize}

\section{Related Work}

We position LGS along three axes: (i) latent vs.\ ambient dynamics, (ii) deterministic vs.\ probabilistic transitions, (iii) single-system vs.\ shared evolution. Prior work covers segments individually; LGS combines all three with an input-noised flow-forcing training rule designed for long autoregressive rollout across heterogeneous PDE families.

\paragraph{Ambient deterministic operators.} Spectral / grid-based neural PDE solvers~\citep{li2020fourier,lu2021learning,li2024physics,li2022transformer,alkin2024universal,kovachki2023neural,azizzadenesheli2024neural,kramer2024learning} and Transformer-style multi-system pretraining~\citep{yang2023context,cao2024vicon,mccabe2023multiple,hao2024dpot,liu2024prose,sun2025towards,lorsung2024physics,zhu2025pimfmphysicsinformedmultimodalfoundation,herde2024poseidon} deliver strong one-step accuracy on the discretized field. Autoregressive error bounds can permit geometric amplification when the dynamics are expansive~\citep{lippe2023pderefinerachievingaccuratelong,ye2025recurrentneuraloperatorsstable}. LGS predicts on a learned latent representation and replaces the deterministic map with a generative transition.

\paragraph{Single-system generative solvers.} Diffusion / flow-matching PDE solvers stay tied to a single system: inverse-reconstruction methods~\citep{huang2024diffusionpdegenerativepdesolvingpartial,du2024conditional,yao2025guideddiffusionsamplingfunction,wang2025fundiffdiffusionmodelsfunction} guide diffusion to recover fields from sparse measurements; spatiotemporal samplers~\citep{lippe2023pderefinerachievingaccuratelong,hu2025waveletdiffusionneuraloperator,feng2025fluidzero,oommen2025integratingneuraloperatorsdiffusion,mucke2025physicsaware} model the next state of one family at short horizons. LGS retains the generative prior, drops the per-system constraint, and adds an explicit long-horizon mechanism.

\paragraph{Latent reduced-order models.} CROM~\citep{lee2020model} and DiNo~\citep{yin2023continuous} pioneered latent dynamics, but each fits a \emph{per-system} encoder and evolves it deterministically. Concurrent single-system latent flow-matching solvers~\citep{li2025generativelatentneuralpde} adopt a similar latent-generative recipe but train on one PDE family at a time without an explicit long-horizon mechanism. LGS preserves the latent-space premise, replaces the deterministic ODE with a probabilistic flow with input-noised flow forcing, and trains a \emph{shared} encoder across $12$ PDE families.

\paragraph{Latent generative pipelines + diffusion forcing.} DiT~\citep{peebles2023scalablediffusionmodelstransformers} and SiT~\citep{ma2024sitexploringflowdiffusionbased} are single-shot encode-then-generate pipelines; their one-step quality bound says nothing about whether an iterated rollout stays on the data manifold. Diffusion forcing~\citep{chen2024diffusionforcingnexttokenprediction,arora2023exposurebiasmattersimitation,xie2025progressive,zhou2024upscale,gao2024autoregressivemovingdiffusionmodels,gao2025ca2vdmefficientautoregressivevideo,huang2025selfforcingbridgingtraintest} addresses analogous exposure bias in autoregressive video and language models by training on partially denoised histories. We adapt this idea to PDE dynamics in latent space. Our analysis in \cref{sec:errcompare_main} gives a conditional rollout-error bound under an explicit stability assumption, complementing the empirical evaluation of input-noised flow forcing.

\paragraph{Multi-scale regular-grid prediction.} PDE-Transformer~\citep{holzschuh2025pdetransformer} is a close deterministic, multi-scale regular-grid surrogate. LGS studies a compressed shared interface with a generative transition and input noising. Different datasets and training protocols preclude an unmatched performance ranking.

\section{Latent Generative Solver}
\label{methods}

\emph{LGS factors the long-horizon problem into a sharable spatial encoder and a generative one-step transition with autoregressively updated context.} Throughout, we distinguish two spaces. The \emph{ambient space} $\mathcal{X}\!\subset\!\mathbb{R}^{H\times W\times C}$ is where raw physical fields $\mathbf{X}_s$ live at the simulation grid: high-dimensional and discretization-specific. The \emph{learned latent representation} $\mathcal{M}\!\subset\!\mathbb{R}^{16\times 16\times 16}$ is the codomain of the PhyVAE encoder, where each latent $\mathbf{x}_s$ resides and the PFlowFT dynamics evolve: low-dimensional and shared across PDE families. This representation is a common tensor interface, without an assumed differential-geometric structure. PhyVAE places latents on a low-dimensional support shared across PDE families; the per-trajectory context $c_s$ identifies which system the trajectory comes from. Shared substrate and system-clustered contexts are complementary. We start with autoregressive one-step prediction over a fixed-length history window,
\begin{equation}
    p(\mathbf{X}_{s+1}\mid \mathbf{X}_{0:s}),\qquad s=0,1,2,\dots,
\end{equation}
where $\mathbf{X}_s\!\in\!\mathcal{X}\!\subset\!\mathbb{R}^{H\times W\times C}$ is the physical state at step~$s$; we use a 4-frame history $\mathbf{X}_{s-3:s}$ throughout. Each transition $s\!\to\! s{+}1$ also uses a transport time $t\!\in\![0,1]$ for flow-based generation; $s$ indexes physical time, $t$ the flow.

\paragraph{Where the design effort goes.} The scaffolding (VAE encode plus latent transformer generation) is shared with single-shot latent generative models such as DiT~\citep{peebles2023scalablediffusionmodelstransformers} and SiT~\citep{ma2024sitexploringflowdiffusionbased}; SiT is our PFlowFT backbone. Our contribution is the \emph{time-forward} method running inside it, addressing failure modes specific to iteration: (i) input-noised training to expose the predictor to perturbed latent states (\cref{sec:errcompare_main}); (ii) flow forcing that re-anchors conditioning on model-generated states; (iii) \rev{spatial-pyramid compression of the fixed history window}.

\begin{figure}[t]
    \centering
    \includegraphics[width=0.82\linewidth]{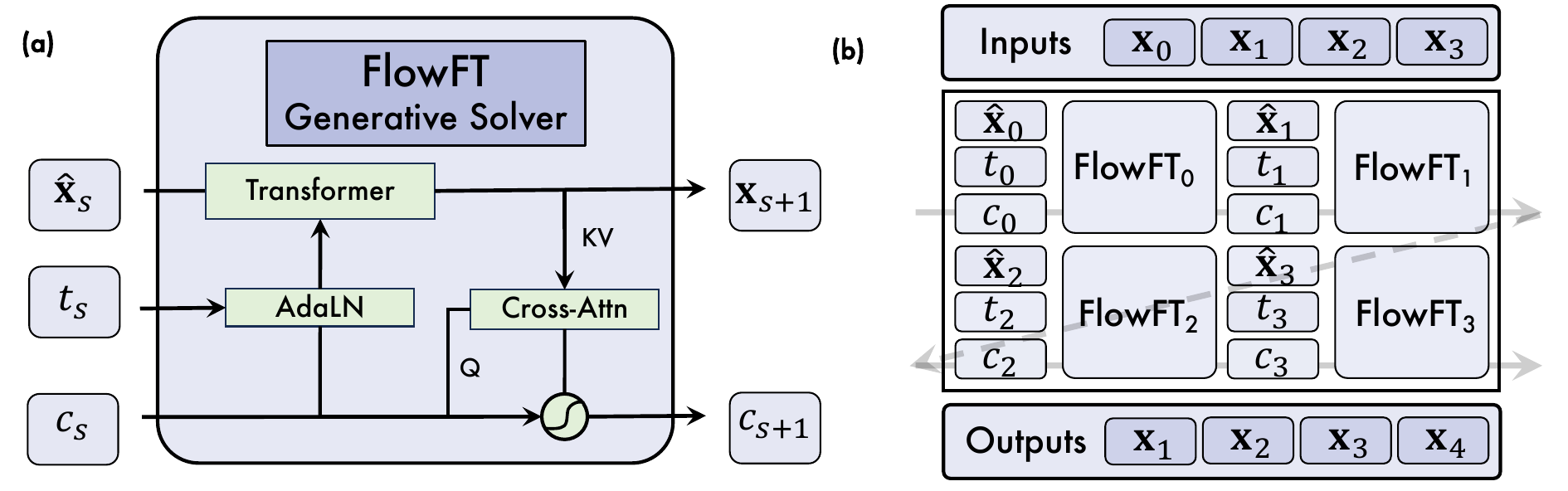}
    \caption{\textbf{FlowFT, the core block of PFlowFT} (before pyramidal context aggregation, \S\ref{sec:method:flow}). (a) A single FlowFT block predicts $\mathbf{x}_{s+1}$ from the current latent $\hat{\mathbf{x}}_s$, transport time $t_s$, and per-trajectory context $c_s$. (b) \rev{Context is constructed within each predictor call from the four-frame window and rebuilt after the window shifts.}}
    \label{fig:model}
\end{figure}

\subsection{Learning Latent Physics with PhyVAE}

A Physics Variational Autoencoder (PhyVAE), built on the SD-VAE backbone~\citep{rombach2022highresolutionimagesynthesislatent}, maps raw fields $\mathbf{X}\!\in\!\mathbb{R}^{H\times W\times C}$ to a $12\times$ smaller latent state $\mathbf{x}\!\in\!\mathbb{R}^{16\times 16\times 16}$ shared across all $12$ PDE families. \rev{Both encoder and decoder are frozen during PFlowFT training;} the decoder is used only for reconstruction and evaluation. Latent trajectories are precomputed and cached, removing per-step encoding from solver training and making long-horizon autoregressive learning tractable.

The representation provides a common tensor interface across standardized channel formats and resolutions. Separating reconstruction from temporal prediction allows PFlowFT to train on pooled latent trajectories while the PhyVAE remains frozen. Reconstruction, latent-capacity ablations, and adaptation to one held-out Kolmogorov system evaluate this interface; they do not establish that nearby latent points always correspond to physically admissible states.

\subsection{Generative Latent Transition}\label{sec:method:flow}

The latent transition must handle two distinct stability bottlenecks over long rollouts. We separate them by factorizing through a per-trajectory context $c_s$ inferred from history:
\begin{equation}
\label{eq:factor_main}
    p(\mathbf{x}_{s+1}\mid \mathbf{x}_{0:s})
    =\int p(\mathbf{x}_{s+1}\mid \mathbf{x}_s,c_s)\,p(c_s\mid \mathbf{x}_{0:s-1})\,dc_s.
\end{equation}
The factorization exposes (i) \emph{state drift} in $\mathbf{x}_s$, which is itself model-generated under autoregressive rollout and drifts off the clean training distribution; and (ii) \emph{context drift} in $c_s$, from finite history and ambiguity over which PDE family the trajectory belongs to. We address (i) with a flow-matching transition equipped with input noising, and (ii) with a flow-forcing scheme that re-anchors $c_s$ on model-generated states. Together they form the Flow-Forcing Transformer (FlowFT).

\paragraph{Flow matching with input noising.}
We model $p(\mathbf{x}_{s+1}\mid \mathbf{x}_s,c_s)$ via a conditional probability flow, suppressing $s$ and writing $\mathbf{x}_0\!\equiv\!\mathbf{x}_s,\mathbf{x}_1\!\equiv\!\mathbf{x}_{s+1}$ for the consecutive pair within one transition; subscripts $0,1$ denote flow endpoints, not physical-time indices. At test time $\mathbf{x}_0$ carries rollout error, so we train on a noised input distribution. Define the $k$-noised input
\begin{equation}
\label{eq:softsource_main}
    \tilde{\mathbf{x}}_0^{\,k}\coloneq (1-k)\mathbf{x}_0+k\mathbf{z},\qquad
    \mathbf{z}\sim\mathcal{N}(0,\mathbf{I}),
\end{equation}
where $k$ is the per-step \emph{noise level}. At every training step, $k$ is sampled afresh from $\mathcal{U}(0,k_{\max})$, where $k_{\max}\!\in\![0,1]$ is a model-level hyperparameter. Setting $k_{\max}\!=\!1$ exposes the model to the full range from clean to fully-noised inputs; it does not feed pure noise on every step. Bridge states are $\mathbf{x}_t^{\,k}\!\coloneq\!(1{-}t)\tilde{\mathbf{x}}_0^{\,k}+t\mathbf{x}_1$ for $t\!\sim\!\mathcal{U}(0,1)$; their sample-wise velocity is $\mathbf{u}_t^{\,k}\!=\!(\mathbf{x}_1\!-\!\mathbf{x}_t^{\,k})/(1{-}t)$. We instantiate the transition as a learned network $\mathbf{g}_\theta$ that predicts the endpoint $\hat{\mathbf{x}}_1\!=\!\mathbf{g}_\theta(\mathbf{x}_t^{\,k},t,c)$ rather than the velocity (the JiT parameterization of~\citealp{li2026jit}); the induced velocity is $\mathbf{v}_\theta\!\coloneq\!(\hat{\mathbf{x}}_1\!-\!\mathbf{x}_t^{\,k})/(1{-}t)$, giving
\begin{equation}
\label{eq:flowloss_main}
    \mathcal{L}_{\mathrm{FM}}
    =\mathbb{E}\big\|\mathbf{v}_\theta-\mathbf{u}_t^{\,k}\big\|_2^2
    =\mathbb{E}\big\|(\hat{\mathbf{x}}_1\!-\!\mathbf{x}_1)/(1{-}t)\big\|_2^2,
\end{equation}
clamping $(1{-}t)\!\ge\!0.05$ for stability. At inference we integrate the probability-flow ODE \emph{from the clean autoregressive output, with no further noising}:
\begin{equation}
\label{eq:pfode_main}
    \tfrac{d\mathbf{x}}{dt}=\mathbf{v}_\theta(\mathbf{x},t,c),\quad \mathbf{x}(0)=\mathbf{x}_0,
\end{equation}
from $0$ to $1{-}\epsilon$ and use the endpoint as the predicted next state. Training with $k\!>\!0$ exposes the model to a $k$-thickened input distribution. This training exposes the predictor to perturbed inputs; no explicit noising is reintroduced at inference. Effective contraction is a separate assumption in \cref{sec:errcompare_main}.

\paragraph{Flow forcing on the data manifold.}
The context evolves as a filtering process $c_s\!\sim\!p(c_s\mid c_{s-1},\mathbf{x}_s)$. At test time this update depends on model-generated $\hat{\mathbf{x}}_s$; training it only on ground-truth states would therefore create exposure bias. We reduce this mismatch by updating $c_{s+1}$ from the predicted state through gated cross-attention (\cref{fig:model}a):
\begin{equation}
\label{eq:condupdate}
    c_{s+1} = \mathrm{Gate}\!\left(c_s,\,\mathrm{Attn}(\hat{\mathbf{x}}_{s+1},c_s)\right),
\end{equation}
so all subsequent context during training is driven by model-generated trajectories. t-SNE projections of the inferred $c_s$ (\cref{fig:conditions}) \rev{illustrate system-clustered contexts with limited drift}. \rev{PFlowFT uses a fixed four-frame window represented, oldest to newest, at $2{\times}2$, $4{\times}4$, $8{\times}8$, and $16{\times}16$ latent resolution (340 tokens). Context is discarded after prediction and rebuilt from the shifted window; the pyramid is an efficiency mechanism rather than persistent memory.} The resulting model is the Pyramidal Flow-Forcing Transformer (PFlowFT).

\subsection{Training and Inference}

We adopt a two-stage training protocol.

\paragraph{Stage I: Latent Physics Encoding.}
PhyVAE learns a compact latent physics space by reconstructing individual system states while regularizing the latent distribution toward a standard normal prior. After convergence, the encoder is frozen and used to precompute all latent trajectories.

\paragraph{Stage II: Generative Dynamics Learning.}
PFlowFT is trained on cached latent trajectories with sequential flow forcing. Per rollout step $s$ in a trajectory: sample $t\!\sim\!\mathcal{U}(0,1)$, $k\!\sim\!\mathcal{U}(0,k_{\max})$, and $\mathbf{z}\!\sim\!\mathcal{N}(0,\mathbf{I})$; form the bridge $\mathbf{x}_t^k=(1{-}t)\big((1{-}k)\mathbf{x}_s+k\mathbf{z}\big)+t\,\mathbf{x}_{s+1}$; predict $\hat{\mathbf{x}}_{s+1}=\mathbf{g}_\theta(\mathbf{x}_t^k,t,c_s)$; accumulate the FM loss \eqref{eq:flowloss_main}; update $c_{s+1}$ via \eqref{eq:condupdate}. The bound $k_{\max}$ is a model hyperparameter (default $1.0$). Gradients pass through $\hat{\mathbf{x}}_{s+1}$ in the context update; older history is pyramid-pooled into $c_s$ for bounded attention cost.

\paragraph{Inference and Rollout.}
Integrate the PF-ODE \eqref{eq:pfode_main} from the current latent (no noising) using 20 Euler evaluations to predict the next state; decode predicted latents to ambient space with the PhyVAE decoder.

\subsection{A Sufficient-Condition Stability Analysis}
\label{sec:errcompare_main}

We analyze error propagation under an explicit stability assumption on the ideal transition. Input-noised training motivates this analysis, but does not by itself guarantee the assumption. The result compares learned and ideal exact-flow rollouts with the same context sequence; Appendix~\ref{erracc2_multistep} additionally accounts for context differences and numerical integration.

\paragraph{Deterministic compounding.}
For a deterministic operator $f_\theta$ approximating $L$-Lipschitz latent dynamics $\Phi$, let $\delta_s=\mathbf{x}_s-\mathbf{x}_s^*$ and $e_s=f_\theta(\mathbf{x}_s)-\Phi(\mathbf{x}_s)$. Then
$\|\delta_{s+1}\|\le L\|\delta_s\|+\|e_s\|$.
For identical initial states, this gives
$\|\delta_s\|\le\sum_{j<s}L^{s-1-j}\|e_j\|$.
When $L>1$, the upper bound permits geometric amplification; it does not require actual error growth (Appendix~\ref{erracc1}).

\paragraph{Probability-flow transition.}
Let $\mathcal{T}_\theta^c$ and $\mathcal{T}^c$ denote the learned and ideal PF-ODE maps over $t\in[0,1-\epsilon]$, for fixed $\epsilon>0$. Both start from an unnoised input at inference. Define the conditional comparison by
$\mathbf{x}_{s+1}=\mathcal{T}_\theta^{c_s}(\mathbf{x}_s)$ and
$\mathbf{x}_{s+1}^*=\mathcal{T}^{c_s}(\mathbf{x}_s^*)$.

\begin{assumption}[Effective stability of the ideal transition]
\label{ass:effective_main}
\label{lem:softening_main}
Uniformly over contexts and states in the rollout domain, the ideal transition satisfies
\[
\|\mathcal{T}^{c}(\mathbf{x})-\mathcal{T}^{c}(\mathbf{y})\|
\le\rho\|\mathbf{x}-\mathbf{y}\|,\qquad 0\le\rho<1.
\]
\end{assumption}

Appendix~\ref{erracc2_multistep} derives a sufficient one-sided flow-field condition for this assumption. An expression $\rho=L_T(1-k)$ is valid only when that effective bound is independently satisfied; shared-noise geometry alone does not establish it for unnoised inference.

\begin{theorem}[Conditional multi-step rollout stability]
\label{thm:multistep_main}
Under Assumption~\ref{ass:effective_main}, suppose the one-step learned-to-ideal deviation satisfies
$\|\mathcal{T}_\theta^{c_s}(\mathbf{x}_s)-\mathcal{T}^{c_s}(\mathbf{x}_s)\|\le B_s$.
Then
\begin{equation}
\label{eq:flow_multistep_main}
\|\delta_{s+1}\|\le\rho\|\delta_s\|+B_s.
\end{equation}
If $B_s\le\overline B$ uniformly, then
$\|\delta_n\|\le\rho^n\|\delta_0\|+\overline B(1-\rho^n)/(1-\rho)$,
so $\limsup_{n\to\infty}\|\delta_n\|\le\overline B/(1-\rho)$.
\end{theorem}

The recursion follows by adding and subtracting $\mathcal{T}^{c_s}(\mathbf{x}_s)$. Under the scaled-residual regularity and ideal-field stability conditions in Appendix~\ref{erracc2_onestep}, Lemma~\ref{lem:rescale_app} bounds $B_s$ using the velocity residual $\mathbf{r}=\mathbf{v}_\theta-\mathbf{v}^*$. The same-context result does not by itself bound error against a physical reference with its own reconstructed context. For the four-frame reconstruction, the appendix gives a sufficient condition $\rho+L_cL_C<1$, where $L_c$ and $L_C$ bound transition sensitivity to context and context sensitivity to history, respectively.

\paragraph{Relation to input noising.}
Training samples $k\sim\mathcal{U}(0,k_{\max})$; $k_{\max}$ controls a mixture of noise levels, not a fixed inference-time contraction factor. The reported sweep reduces 10-step L2RE from $23.5\%$ to $18.3\%$ as $k_{\max}$ increases across the tested values (Table~\ref{tab:k_sweep}). This supports the empirical benefit of input noising; it does not measure $\rho$ or verify Assumption~\ref{ass:effective_main}.

\section{Experiment Setup}
\label{benchmark}

\subsection{General setup}

\paragraph{Dataset gathering.}

We do not introduce new simulations; instead, we standardize four existing public collections into a \rev{common evaluation protocol}: FNO-v~\citep{li2020fourier}, PDEArena~\citep{gupta2022multispatiotemporalscalegeneralizedpdemodeling}, PDEBench~\citep{takamoto2024pdebenchextensivebenchmarkscientific}, and The Well~\citep{ohana2025welllargescalecollectiondiverse}. The standardization spans (i) spatial resolution (resampled to $128^2$), (ii) channel format (unified to 3 multiphysics channels, c3p128), (iii) numerical precision (float16), (iv) train/val/test splits (8:1:1), and (v) a shared evaluation protocol (1/5/10/20-step L2RE). The result is a 233\,GB corpus over 16 systems and 12 PDE families on which deterministic and generative solvers can be compared at matched active-parameter and predictor-update budgets. Per-source compression ratios and the c3p128 channel mapping are in Appendix~\ref{dataset} and \cref{tab:channel_map}.

\paragraph{Implementation.}
PhyVAE uses the SD-VAE backbone~\citep{rombach2022highresolutionimagesynthesislatent} to compress each state from c3p128 to c16p16 ($12\times$ compression; channel-count ablation in Table~\ref{tab:vae_ablation}). PFlowFT uses RMSNorm, SwiGLU, and FlashAttention~v2~\citep{dao2023flashattention2fasterattentionbetter} inside an SiT~\citep{ma2024sitexploringflowdiffusionbased} (Transformer) backbone. PhyVAE ($249$\,M) is trained for $100$\,k steps with KL weight $\beta\!=\!10^{-3}$; PFlowFT ($138$\,M) is then trained for $200$\,k steps on cached latent trajectories. Full hyperparameters appear in Appendix~\ref{apx:imp:lgs}.
\paragraph{Evaluation.}
We report relative $\ell_2$ error (L2RE) at 1-, 5-, 10-, and 20-step rollouts. \rev{Appendix~\ref{apx:tables} additionally reports spectral and temporal diagnostics.} Per-system L2RE is sample-averaged on the held-out test split. Table averages are uniform over the 16 systems, \rev{except the 20- and 40-step averages, which use the 11 eligible systems}, with subscripts denoting the sample standard deviation across systems (not run-to-run variation). ``Active dynamics parameters'' refers to parameters used by the recurrent dynamics model at each rollout step; LGS additionally uses a frozen 249\,M-parameter PhyVAE. \cref{fig:flops} reports \rev{training-time} recurrent dynamics-step FLOPs only. The PhyVAE encoder ($\sim$70 GFLOPs/frame) and decoder ($\sim$113 GFLOPs/frame) are excluded; the per-step decoder cost should be added when downstream tasks require decoded fields at every step.

\subsection{Baselines}

We compare against three deterministic baselines: U-AFNO~\citep{Bonneville_2025} (FNO family), CNextUNet (a ConvNeXt-block U-Net~\citep{liu2022convnet,ronneberger2015unetconvolutionalnetworksbiomedical}, U-Net family), and DPOT~\citep{hao2024dpot} (Transformer family), each at $\sim$150\,M active parameters to match LGS. To avoid the training-protocol mismatch reported across prior benchmarks, we re-train every baseline ourselves under a single configuration: the same 4-frame history window (\cref{methods}, channel-concatenated for U-Net/FNO baselines), 200\,k optimizer steps, and an identical predictor-training schedule. The PhyVAE pretraining adds 100\,k updates before the 200\,k PFlowFT updates. To rule out parameter-count effects, we additionally scale U-AFNO and DPOT to $\sim$400\,M active parameters (matching the combined PhyVAE\,+\,LGS count) and report results in Appendix~\ref{apx:imp:baseline} (\cref{tab:scaled_baselines}). We also adapt two recent generative PDE methods, DiffusionPDE~\citep{huang2024diffusionpdegenerativepdesolvingpartial} and FunDiff~\citep{wang2025fundiffdiffusionmodelsfunction}, to our autoregressive protocol (\cref{tab:generative_baselines}).

\begin{table}[ht]
\centering
\small
\caption{\textbf{Headline result.} Multi-horizon average L2RE (lower the better, \%) over 16 PDE systems at matched $\sim$150\,M active parameters; \rev{the 20-step column uses the 11 eligible systems}. Subscripts give cross-system standard deviation. \textsc{Wins/16}: per-horizon count of systems on which a model attains the lowest L2RE. Bold marks the column-best mean.}
\label{tab:lgs_results}
\begin{tabular}{lcccc|cc}
\toprule
\textsc{Model} & 1-step & 5-step & 10-step & 20-step
& \multicolumn{2}{c}{\textsc{Wins / 16}} \\
& & & & & 5-step & 10-step \\
\midrule
CNextUNet (153\,M) & $20.8_{\pm 22.0}$ & $42.9_{\pm 20.1}$ & $58.6_{\pm 22.3}$ & $74.2_{\pm 27.9}$ & 0 & 0 \\
DPOT (160\,M)       & $6.1_{\pm 4.6}$ & $16.9_{\pm 9.0}$ & $29.6_{\pm 15.8}$ & $60.9_{\pm 32.6}$ & 1 & 1 \\
U-AFNO (152\,M)     & $\mathbf{5.0}_{\pm 3.7}$ & $17.0_{\pm 9.2}$ & $29.6_{\pm 15.7}$ & $56.1_{\pm 18.1}$ & 0 & 0 \\
LGS (138\,M dyn.\,+\,249\,M VAE) & $6.0_{\pm 3.6}$ & $\mathbf{12.1}_{\pm 6.4}$ & $\mathbf{18.3}_{\pm 10.1}$ & $\mathbf{30.2}_{\pm 18.3}$ & \textbf{15} & \textbf{15} \\
\bottomrule
\end{tabular}
\end{table}

\begin{table}[h]
\caption{Adapted generative baselines: mean L2RE (\%) over 16 systems.}
\label{tab:generative_baselines}
\centering
\small
\begin{tabular}{llccc}
\toprule
Model & Type & 1-step & 5-step & 10-step \\
\midrule
DiffusionPDE (adapted) & Diffusion (EDM) & 7.5 & 25.4 & 42.5 \\
FunDiff (adapted)      & Function-space  & 10.4 & 29.1 & 44.8 \\
LGS                    & Flow matching   & 6.0 & 12.1 & 18.3 \\
\bottomrule
\end{tabular}
\end{table}

\subsection{Ablation Studies}
\label{sec:ablation}

We ablate five ingredients of PFlowFT to identify which drive (i) long-horizon stability and (ii) efficiency. The ingredients are: (a) the generative transition $p(\mathbf{x}_{s+1}\mid\mathbf{x}_s,c_s)$, (b) flow-forcing supervision over every step in $\mathbf{x}_{1:s+1}$, (c) the noise level $k$ (Eq.~\eqref{eq:softsource_main}), (d) context maintenance and update $c_{s+1}\!=\!\mathrm{Update}(c_s,\hat{\mathbf{x}}_{s+1})$ (Eq.~\eqref{eq:condupdate}), and (e) temporal pyramids. All variants share the frozen PhyVAE and a context length of 4. We report two ablation sweeps: an \emph{aggregated} chain A1$\to$A4 that removes (e)$\to$(a) cumulatively, and a \emph{single-component} sweep \{A2S, A3S, A4S\} that removes exactly one ingredient each (Table~\ref{tab:ablation_single}). Implementation in Appendix~\ref{apx:imp:ablation}.

\FloatBarrier

\section{Experiment Results}
\label{experiments}

\begin{figure}[t]
    \centering
    \includegraphics[width=\linewidth]{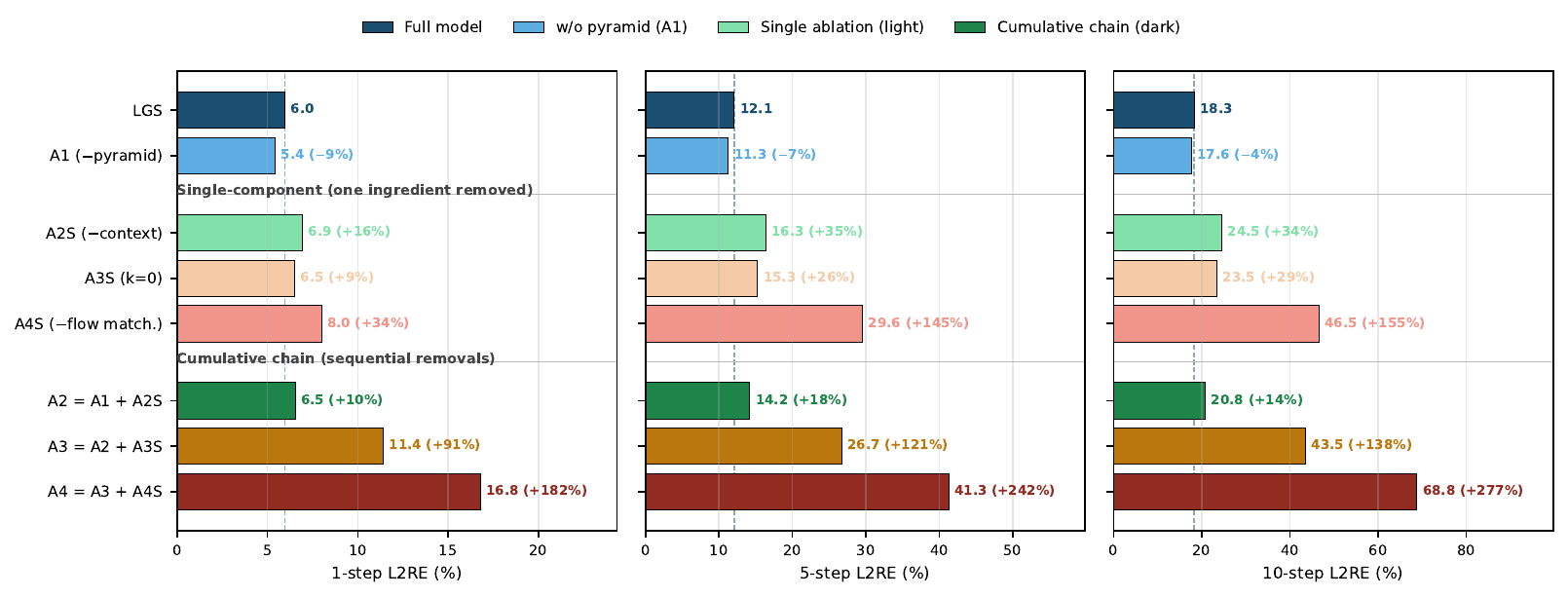}
    \caption{\textbf{Single-component vs.\ cumulative ablation} (1/5/10-step L2RE, $\downarrow$). Light bars: one ingredient removed at a time (A1 no pyramid; A2S no context; A3S $k{=}0$; A4S no flow matching). Dark bars: cumulative chain A2$=$A1$+$A2S, A3$=$A2$+$A3S, A4$=$A3$+$A4S. Percentages are relative increases over full LGS at each horizon.}
    \label{fig:ablation}
\end{figure}

\subsection{Ablation Analysis}\label{sec:exp_ablation}

\emph{Flow matching is the load-bearing component; $k$ and context are stability scaffolds; the pyramid is an efficiency dial.} \cref{fig:ablation} sweeps single-component removals (light) against the cumulative chain (dark) across three horizons. Removing flow matching alone (A4S) inflates 5/10-step L2RE by $+145$/$+155\%$, about $5\times$ beyond $k$ ($+26$/$+29\%$) or context ($+35$/$+34\%$). These ablations measure the contribution of each training mechanism; they do not independently verify the effective-contraction assumption in \cref{thm:multistep_main}. \emph{$k$ and context are stability scaffolds, not accuracy drivers}: each costs at most $16\%$ at 1-step but $+29$/$+34\%$ at 10-step, so a 1-step protocol would underrate both. \emph{The chain is superadditive}: A4 reaches $+277\%$ versus $+218\%$ summed, so the three generative pieces are coupled rather than independent. The temporal pyramid trades a small increase in error for lower compute: full-resolution history improves 10-step L2RE from $18.3\%$ to $17.6\%$ but requires approximately four times the predictor FLOPs (Table~\ref{tab:pyramid_control}). Per-system breakdowns of every ablation are in Appendix~\cref{tab:ablation}.

Per-system breakdowns (Appendix~\cref{tab:ablation}) show A4S degrades chaotic systems (W-SWE 10.5\%~$\to$~98.3\% at 10-step) far more than well-behaved ones (PB-CNSL 5.1\%~$\to$~7.3\%). These comparisons identify systems where removing the generative transition causes larger increases in rollout error.

\begin{table}[t]
\centering\small
\caption{Per-system L2RE (\%) on the same 11 eligible systems at 20 and 40 steps. Values are reproduced from Tables~\ref{tab:rollout} and~\ref{tab:forty_step}.}
\label{tab:long_main}
\begin{tabular}{lrrrr}
\toprule
& \multicolumn{2}{c}{20 steps} & \multicolumn{2}{c}{40 steps} \\
System & LGS & U-AFNO & LGS & U-AFNO \\
\midrule
FNO-v3 & 16.0 & 41.7 & 50.8 & 85.0 \\
FNO-v4 & 38.4 & 83.1 & 51.4 & 129.0 \\
PA-NS & 31.7 & 41.6 & 35.6 & 54.0 \\
PA-SWE & 60.3 & 77.5 & 71.0 & 101.6 \\
PB-SWE & 22.0 & 76.5 & 26.3 & 115.4 \\
W-AM & 65.3 & 57.8 & 58.4 & 71.7 \\
W-GS & 25.2 & 33.2 & 21.8 & 58.6 \\
W-SWE & 10.8 & 69.4 & 19.6 & 48.1 \\
W-RB & 33.0 & 33.8 & 35.8 & 71.2 \\
W-SF & 12.7 & 53.6 & 11.0 & 73.8 \\
W-TR & 17.0 & 48.6 & 28.2 & 28.8 \\
\bottomrule
\end{tabular}
\end{table}

\subsection{Long-term Rollout Stability}\label{sec:exp_long}

\begin{figure}[t]
\centering
\begin{minipage}[c]{0.50\linewidth}
\centering
\includegraphics[width=\linewidth]{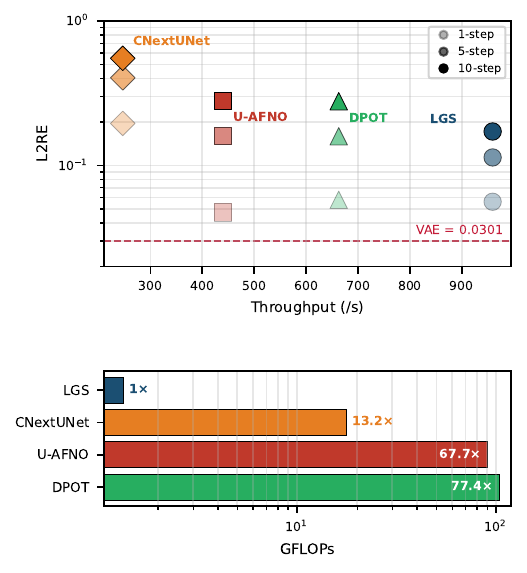}
\end{minipage}\hfill
\begin{minipage}[c]{0.46\linewidth}
\caption{\textbf{Accuracy, throughput, computation.} \emph{Top:} L2RE vs.\ throughput (H200, batch 64); marker shade encodes rollout horizon (1/5/10-step); dashed line is the PhyVAE reconstruction floor at $0.0301$. \emph{Bottom:} \rev{training-time} per-sample recurrent dynamics-step GFLOPs relative to LGS (PhyVAE encode/decode excluded, see~\cref{benchmark}); matched deterministic baselines run at $13$--$77\times$ LGS's compute. Deployment requires 20 predictor evaluations per transition; full-transition accounting is in Table~\ref{tab:profiling}.}
\label{fig:flops}
\end{minipage}
\end{figure}
\cref{tab:lgs_results} summarizes the main long-horizon result. LGS obtains $30.2\%$ L2RE versus U-AFNO's $56.1\%$ at 20 steps, lowest L2RE on 15 of 16 systems at 5 and 10 steps, with per-system comparisons at 20 and 40 steps in Table~\ref{tab:long_main}. \rev{Tables~\ref{tab:long_main} and~\ref{tab:generative_baselines} report per-system long-horizon results and the adapted generative comparison; full rollout tables remain in Appendix~\ref{apx:tables}.} Against adapted generative baselines, LGS's 10-step L2RE is $18.3\%$ vs.\ DiffusionPDE $42.5\%$ / FunDiff $44.8\%$. \cref{fig:flops} places this against \rev{training-time} per-step compute: latent-space models use $13$--$77\times$ less recurrent dynamics-step compute than matched deterministic baselines (encoder/decoder excluded, see \cref{benchmark}).

On the same 11-system population, 40-step L2RE is $37.3\%$ for LGS and $76.1\%$ for U-AFNO; LGS has lower error on all 11 systems (Table~\ref{tab:long_main}). \cref{fig:steps_show}a makes the contrast concrete on a representative FNO-v3 ten-step rollout, where LGS preserves coherent vortex structure while U-AFNO accumulates visible artifacts. The gap is not a parameter-count effect: scaling U-AFNO and DPOT to $\sim$400\,M (matching the combined PhyVAE\,+\,LGS count) improves their 1-step accuracy but leaves 10-step gaps of $5.6$ and $1.4$ percentage points for U-AFNO and DPOT, respectively (Appendix~\cref{tab:scaled_baselines,tab:scaled}). Per-system numbers (Appendix~\cref{tab:rollout}) show the same pattern across every PDE family. LGS is also more robust to corrupted inputs: at SNR$=10$\,dB observation noise its 10-step error rises by $0.51$ percentage points versus U-AFNO's $3.99$ (\cref{tab:noise_robustness}); at $50\%$ spatial masking, $30.9\%$ versus $57.3\%$ (\cref{tab:mask_robustness}). Per-system rollout visualizations for every PDE family appear in Appendix~\cref{fig:fno,fig:pa,fig:pb,fig:well,fig:pans}.

\begin{figure}[t]
\centering
\includegraphics[width=\linewidth]{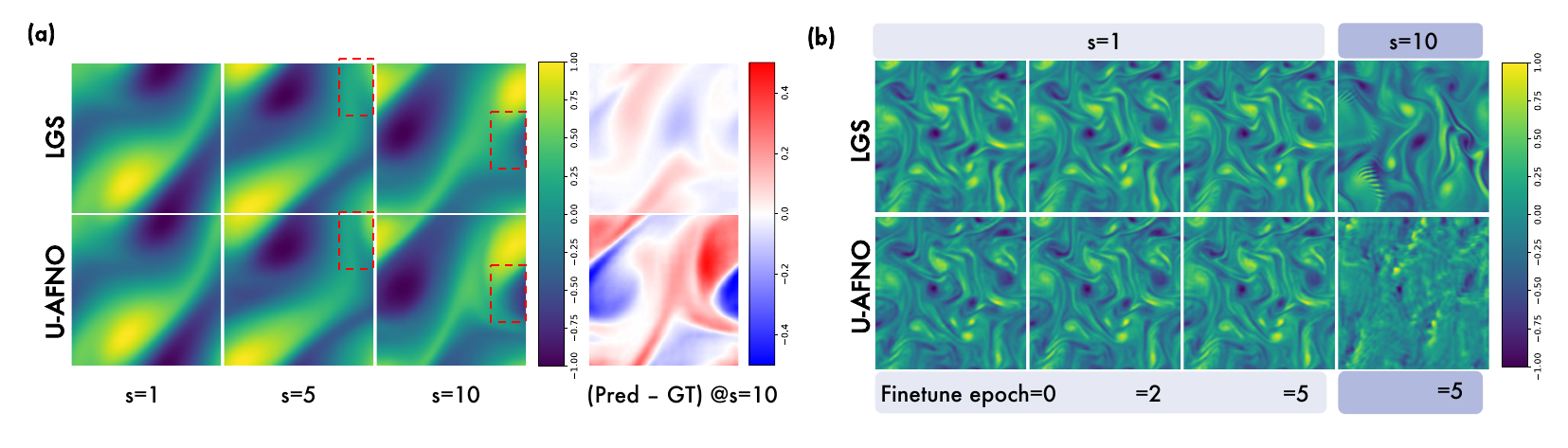}
\caption{(a) FNO-v3 ten-step rollout at $s{=}1,5,10$ with residual at $s{=}10$: U-AFNO accumulates artifacts (red boxes), LGS preserves vortex structure. (b) Finetuning on a held-out $256^2$ Kolmogorov flow: $s{=}1$ rollouts at finetune epochs $0,2,5$ (left three columns), and the converged $s{=}10$ snapshot at epoch $5$ (right). LGS stays stable; U-AFNO blurs and drifts.}
\label{fig:steps_show}
\end{figure}

\FloatBarrier
\subsection{Adaptation to a held-out high-resolution flow}

\begin{wraptable}{r}{0.50\linewidth}
\vspace{-1.0em}
\centering
\small
\caption{\textbf{Adaptation to $256^2$ Kolmogorov flow} ($10\%$ target data, 1-step L2RE, $\downarrow$).}
\label{tab:adapt}
\begin{tabular}{lccc}
\toprule
Model & Zero-shot & Epoch 2 & Epoch 5 \\
\midrule
U-AFNO & 0.6528 & 0.5789 & 0.3430 \\
LGS & \textbf{0.3984} & \textbf{0.1781} & \textbf{0.1291} \\
\bottomrule
\end{tabular}
\vspace{-0.5em}
\end{wraptable}
Both U-AFNO and LGS accept variable-resolution inputs, but they adapt at different rates. We finetune each on a $256^2$ Kolmogorov flow, held out from pretraining and with four times as many grid points as the $128^2$ training fields ($10\%$ target data, five epochs; \cref{tab:adapt}); the result evidences faster adaptation to this single target rather than broad OOD generalization. LGS reduces 1-step error from $0.398$ to $\mathbf{0.129}$, vs.\ U-AFNO's $0.653$ to $0.343$ (about $2.7\times$ lower at epoch 5). Under self-rollout (\cref{fig:steps_show}b), U-AFNO blurs and drifts, while LGS produces a stable Kolmogorov flow that improves monotonically with finetuning.

Finetuning updates the latent dynamics while retaining the shared encoder. The comparison establishes adaptation to this held-out system under the stated five-epoch protocol; finetuning wall-clock time is not reported.

\section{Conclusion}
\label{sec:discussion}

The \textbf{Latent Generative Solver} (LGS) performs autoregressive prediction in a shared learned latent representation. \rev{Across the evaluated 2D regular-grid systems and discrete horizons, it uses $13$--$77\times$ less training-time recurrent dynamics-step compute than matched deterministic operators.} The latent representation reduces the number of scalar state entries by $12\times$, providing a compact domain for the dynamics predictor. \rev{Extension to 3D fields, coupled physics, and irregular domains requires new representations and remains future work.}

\paragraph{Limitations.}\label{sec:limitations} Our experiments are restricted to 2D systems; extension to 3D requires a volumetric PhyVAE. Benchmark standardization to c3p128 may favor methods that exploit latent compression. Further caveats (PhyVAE reconstruction floor, isotropic Gaussian noising near multimodal bifurcations, single-seed reporting) are detailed in Appendix~\ref{apx:add_limitations}.

\begin{ack}
This work was supported by the MIT Generative AI Impact Consortium. We acknowledge computational resources provided by the MIT Office of Research Computing and Data (ORCD) and the Kestrel supercomputer at the National Laboratory of the Rockies (NLR). The Kestrel resources were sponsored by the U.S. Department of Energy's Office of Critical Minerals and Energy Innovation. We also thank Dr.~Haixu Wu of MIT's Computer Science and Artificial Intelligence Laboratory (CSAIL) for the thoughtful discussion.
\end{ack}

\newpage
\bibliographystyle{plainnat}
\bibliography{references}

\newpage
\appendix

\section{Additional Limitations}
\label{apx:add_limitations}

The body lists the two structurally most consequential caveats; we record three further ones here for completeness.

\paragraph{PhyVAE reconstruction floor.} The frozen PhyVAE introduces a $\sim\!3\%$ irreducible reconstruction error (Tab.~\ref{tab:vae_ablation}). All downstream long-horizon accuracy is bounded below by this floor: even a perfect dynamics predictor in the latent space cannot recover sub-$3\%$ L2RE in the ambient space. Joint encoder--dynamics training, or higher-channel latents at the cost of compression, would relax this bound.

\paragraph{Isotropic Gaussian noising.} The input-noised flow-matching distribution adds isotropic Gaussian perturbations to the source latent (Eq.~\ref{eq:softsource_main}). These perturbations are a training surrogate for rollout error. They need not match its distribution, particularly near multimodal bifurcations, and do not establish the effective-contraction assumption.

\paragraph{Single-seed reporting.} Each result reflects a single training run; subscripts in Tab.~\ref{tab:lgs_results} are cross-system standard deviations rather than run-to-run variability. We follow the convention of large-corpus benchmarks~\citep{mccabe2023multiple,herde2024poseidon,hao2024dpot}, where per-system seed sweeps are infeasible at the model and data scale used here.

\section{Dynamics Condition Clustering}
\label{apx:conditions}

LGS infers a per-trajectory context $c$ from history. \cref{fig:conditions} projects sampled contexts via t-SNE: each system forms a tight cluster, and the per-system drift across the flow's transport time is small. \rev{These projections are descriptive evidence of regime separation, not proof of a unique physical interpretation.}

\begin{figure}[h]
    \centering
    \includegraphics[width=0.85\linewidth]{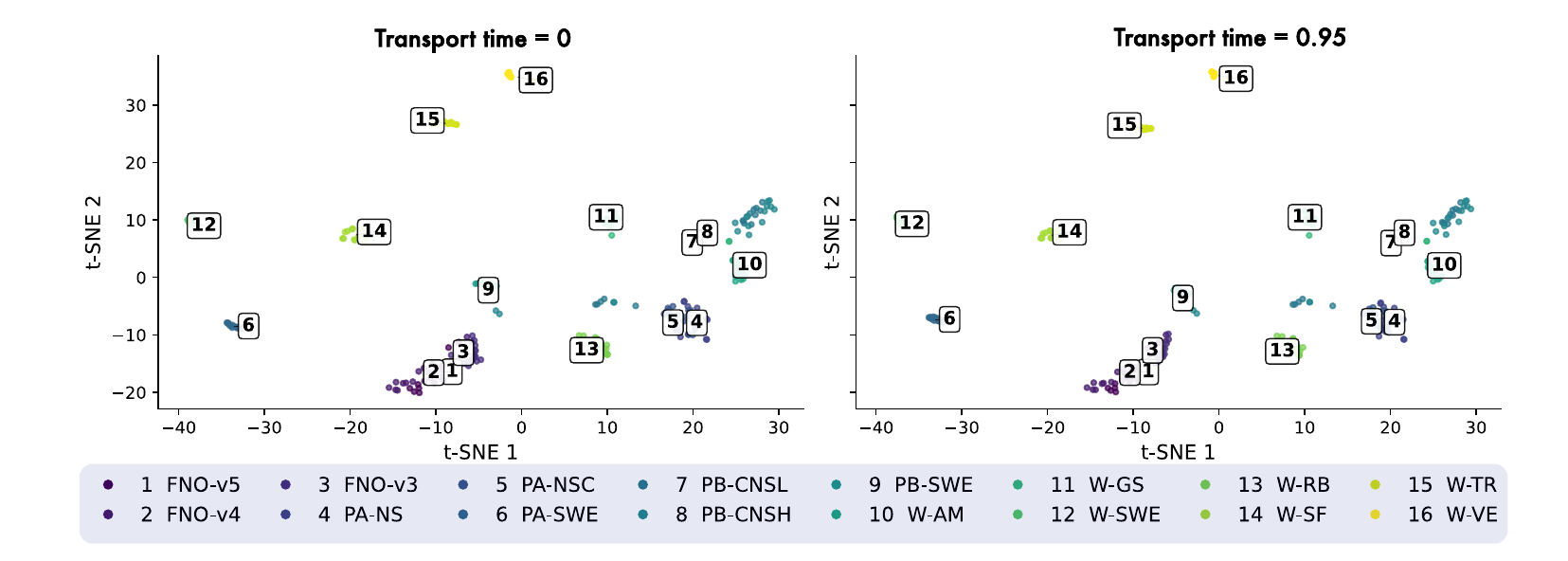}
    \caption{t-SNE projection of inferred per-trajectory contexts $c$ across PDE systems, taken at the start ($t{=}0$) and end ($t{=}1{-}\epsilon$) of the flow-matching transport, sampled at 16 initial conditions. \rev{The clustering is descriptive evidence of regime separation rather than a unique physical interpretation.}}
    \label{fig:conditions}
\end{figure}

\section{Numerical analysis on error accumulation}
\label{erracc}
\setcounter{equation}{0}
\renewcommand{\theequation}{A\arabic{equation}}
\renewcommand{\theHequation}{appendix.\arabic{equation}}

\subsection{Deterministic neural operators: geometric compounding}
\label{erracc1}

\paragraph{Setup.}
Let $\Phi:\mathcal{X}\to\mathcal{X}$ be the true one-step latent dynamics and
$f_\theta:\mathcal{X}\to\mathcal{X}$ a deterministic learned operator. Rollouts satisfy
\begin{equation}
\label{eq:det_rollout_app}
\mathbf{x}_{s+1}^*=\Phi(\mathbf{x}_s^*),\qquad
\mathbf{x}_{s+1}=f_\theta(\mathbf{x}_s).
\end{equation}
Define $\delta_s\coloneq \mathbf{x}_s-\mathbf{x}_s^*$ and
$e_s\coloneq f_\theta(\mathbf{x}_s)-\Phi(\mathbf{x}_s)$.

\begin{assumption}[Lipschitz true dynamics]
\label{ass:lipschitz_app}
$\Phi$ is $L$-Lipschitz: $\|\Phi(\mathbf{x})-\Phi(\mathbf{y})\|\le L\|\mathbf{x}-\mathbf{y}\|$.
\end{assumption}

\begin{lemma}[Deterministic error recursion]
\label{lem:det_rec_app}
Under Assumption~\ref{ass:lipschitz_app},
\begin{equation}
\label{eq:det_rec_app}
\|\delta_{s+1}\|\le \|e_s\|+L\|\delta_s\|.
\end{equation}
\end{lemma}
\begin{proof}
Add and subtract $\Phi(\mathbf{x}_s)$:
\[
\delta_{s+1}
=f_\theta(\mathbf{x}_s)-\Phi(\mathbf{x}_s^*)
=e_s+\big(\Phi(\mathbf{x}_s)-\Phi(\mathbf{x}_s^*)\big),
\]
then apply the triangle inequality and Lipschitzness of $\Phi$.
\end{proof}

\begin{proposition}[Deterministic compounding bound]
\label{prop:det_bound_app}
If $\delta_0=0$, then
\begin{equation}
\label{eq:det_bound_app}
\|\delta_s\|\le \sum_{j=0}^{s-1}L^{\,s-1-j}\|e_j\|.
\end{equation}
\end{proposition}
\begin{proof}
Unroll the recursion \eqref{eq:det_rec_app} by induction.
\end{proof}

\paragraph{Interpretation.}
When $L>1$, the upper bound permits geometric amplification of one-step errors. It is not a lower bound and does not establish that every deterministic rollout exhibits such growth.

\subsection{Multi-step stability under input noising}
\label{erracc2}

\paragraph{Setup}
We model $p(\mathbf{x}_1\mid \mathbf{x}_0,c)$ by a conditional probability flow trained on the
$k$-noised input \eqref{eq:softsource_main}:
\begin{equation}
\label{eq:softsource_app}
\tilde{\mathbf{x}}_0^{\,k}=(1-k)\mathbf{x}_0+k\mathbf{z},\qquad \mathbf{z}\sim\mathcal{N}(0,\mathbf{I}).
\end{equation}
Bridge states are $\mathbf{x}_t^{\,k}=(1-t)\tilde{\mathbf{x}}_0^{\,k}+t\mathbf{x}_1$ and the sample-wise velocity is
$\mathbf{u}_t^{\,k}=(\mathbf{x}_1-\mathbf{x}_t^{\,k})/(1-t)$.
The model predicts $\hat{\mathbf{x}}_1=\mathbf{g}_\theta(\mathbf{x}_t^{\,k},t,c)$ and induces the velocity
\begin{equation}
\label{eq:vtheta_app}
\mathbf{v}_\theta(\mathbf{x},t,c)\coloneq\frac{\mathbf{g}_\theta(\mathbf{x},t,c)-\mathbf{x}}{1-t}.
\end{equation}
At inference, the PF-ODE is integrated \emph{from the rollout output $\mathbf{x}_0$ directly, with no further noising}:
\begin{equation}
\label{eq:pfode_app}
\dot{\mathbf{x}}=\mathbf{v}_\theta(\mathbf{x},t,c),\qquad \mathbf{x}(0)=\mathbf{x}_0,
\end{equation}
from $t=0$ to $t=1-\epsilon$ and $\hat{\mathbf{x}}_{1-\epsilon}$ is taken as the predicted next latent. Training-time noising exposes the model to perturbed inputs. Any contraction of the inference transition requires additional conditions, stated below; it is not a consequence of this exposure alone.

\paragraph{Ideal transport field.}
Let $\mathbf{v}^*(\mathbf{x},t,c)$ denote the (unknown) ideal velocity field corresponding to the true conditional transition.
We measure model error via the residual
\begin{equation}
\label{eq:resid_def_app}
\mathbf{r}(\mathbf{x},t,c)\coloneq \mathbf{v}_\theta(\mathbf{x},t,c)-\mathbf{v}^*(\mathbf{x},t,c).
\end{equation}

\begin{assumption}[Scaled Lipschitz regularity on the rollout support]
\label{ass:scaledlip_app}
Fix $c$. On a domain containing the learned and ideal PF-ODE paths from all initial states used in the comparison, including the unnoised rollout states at inference, the scaled residual $(1-t)\mathbf{r}(\cdot,t,c)$ is Lipschitz in $\mathbf{x}$ with constant $L_r(t)$:
\begin{equation}
\label{eq:scaledlip_app}
\|(1-t)(\mathbf{r}(\mathbf{x},t,c)-\mathbf{r}(\mathbf{y},t,c))\|\le L_r(t)\|\mathbf{x}-\mathbf{y}\|.
\end{equation}
\end{assumption}

\begin{assumption}[Ideal field stabilization near $t\to1$]
\label{ass:idealstab_app}
Fix $c$. The ideal velocity satisfies a one-sided stability bound:
there exists $m(t)\ge 0$ such that for all $\mathbf{x},\mathbf{y}$ on the rollout support,
\begin{equation}
\label{eq:idealstab_app}
\langle \mathbf{x}-\mathbf{y},\,\mathbf{v}^*(\mathbf{x},t,c)-\mathbf{v}^*(\mathbf{y},t,c)\rangle
\le -\frac{1-m(t)}{1-t}\|\mathbf{x}-\mathbf{y}\|^2.
\end{equation}
\end{assumption}

\paragraph{Remarks on assumptions.}
These are explicit regularity and stability conditions, not consequences of endpoint parameterization or input-noised training. In particular, the factor $1/(1-t)$ in Eq.~\eqref{eq:vtheta_app} does not control the sign of the state derivative. The conditions must hold on the inference paths as well as any noised training paths used in a comparison. For fixed $\epsilon>0$, assume the coefficients and residual terms appearing in the integrals below are integrable. Their bounds may depend on $\epsilon$; no uniform $\epsilon\to0$ limit is asserted.

\subsubsection{A one-step deviation bound (pathwise)}
\label{erracc2_onestep}

Fix $c$. Let $\mathbf{x}(t)$ and $\mathbf{x}^*(t)$ be PF-ODE solutions driven by $\mathbf{v}_\theta$ and $\mathbf{v}^*$ with the same initial condition $\mathbf{x}(0)=\mathbf{x}^*(0)=a$ in the comparison domain. The state $a$ may be an unnoised rollout output or a noised training input.
Define $\Delta(t)\coloneq \mathbf{x}(t)-\mathbf{x}^*(t)$.

\begin{lemma}[Rescaled Gronwall bound]
\label{lem:rescale_app}
Under Assumptions~\ref{ass:scaledlip_app}--\ref{ass:idealstab_app}, for any $\epsilon\in(0,1)$,
\begin{equation}
\label{eq:onestep_app}
\|\Delta(1-\epsilon)\|
\le
\epsilon \exp\!\left(\int_0^{1-\epsilon}\frac{L_r(t)+m(t)}{1-t}dt\right)
\int_0^{1-\epsilon}\frac{\|(1-t)\mathbf{r}(\mathbf{x}^*(t),t,c)\|}{(1-t)^2}dt.
\end{equation}
\end{lemma}

\begin{proof}
Differentiate $\Delta$:
\[
\dot{\Delta}(t)=\mathbf{v}^*(\mathbf{x}(t),t,c)-\mathbf{v}^*(\mathbf{x}^*(t),t,c)+\mathbf{r}(\mathbf{x}(t),t,c).
\]
Let $y(t)\coloneq \Delta(t)/(1-t)$. Then
\[
\dot{y}(t)=\frac{\dot{\Delta}(t)}{1-t}+\frac{\Delta(t)}{(1-t)^2}.
\]
Taking the inner product with $y(t)$ and using $\Delta(t)=(1-t)y(t)$ gives
\[
\frac{1}{2}\frac{d}{dt}\|y(t)\|^2
=\left\langle y(t),\frac{\mathbf{v}^*(\mathbf{x}(t),t,c)-\mathbf{v}^*(\mathbf{x}^*(t),t,c)}{1-t}\right\rangle
+\left\langle y(t),\frac{\mathbf{r}(\mathbf{x}(t),t,c)}{1-t}\right\rangle
+\frac{\|y(t)\|^2}{1-t}.
\]
Apply Assumption~\ref{ass:idealstab_app} with $\Delta=(1-t)y$:
\[
\left\langle y,\frac{\mathbf{v}^*(\mathbf{x},t,c)-\mathbf{v}^*(\mathbf{x}^*,t,c)}{1-t}\right\rangle
\le -\frac{1-m(t)}{(1-t)}\|y\|^2.
\]
Thus the ideal field cancels the singular drift up to $m(t)$:
\[
\frac{1}{2}\frac{d}{dt}\|y\|^2
\le \frac{m(t)}{1-t}\|y\|^2+\left\langle y,\frac{\mathbf{r}(\mathbf{x}(t),t,c)}{1-t}\right\rangle.
\]
For the residual term, add and subtract $\mathbf{r}(\mathbf{x}^*(t),t,c)$ and use Cauchy--Schwarz:
\[
\left\langle y,\frac{\mathbf{r}(\mathbf{x}(t),t,c)}{1-t}\right\rangle
\le \frac{\|y\|}{1-t}\|\mathbf{r}(\mathbf{x}^*(t),t,c)\|
+\frac{\|y\|}{1-t}\|\mathbf{r}(\mathbf{x}(t),t,c)-\mathbf{r}(\mathbf{x}^*(t),t,c)\|.
\]
Using Assumption~\ref{ass:scaledlip_app} and $\|\Delta\|=(1-t)\|y\|$,
\[
\|\mathbf{r}(\mathbf{x}(t),t,c)-\mathbf{r}(\mathbf{x}^*(t),t,c)\|
\le \frac{L_r(t)}{1-t}\|\Delta(t)\|
=L_r(t)\|y(t)\|.
\]
Therefore,
\[
\frac{d}{dt}\|y(t)\|
\le \frac{L_r(t)+m(t)}{1-t}\|y(t)\|+\frac{\|\mathbf{r}(\mathbf{x}^*(t),t,c)\|}{1-t}.
\]
Apply Gronwall from $0$ to $1-\epsilon$ and use $y(0)=0$ (same initial condition):
\[
\|y(1-\epsilon)\|
\le
\exp\!\left(\int_0^{1-\epsilon}\frac{L_r(t)+m(t)}{1-t}dt\right)
\int_0^{1-\epsilon}\frac{\|\mathbf{r}(\mathbf{x}^*(t),t,c)\|}{1-t}dt.
\]
Finally $\Delta(1-\epsilon)=\epsilon\,y(1-\epsilon)$ yields \eqref{eq:onestep_app}. Rewriting
$\|\mathbf{r}\|=\|(1-t)\mathbf{r}\|/(1-t)$ gives the stated form.
\end{proof}

\subsubsection{Multi-step rollout recursion and the role of $k$}
\label{erracc2_multistep}

\paragraph{Scope of the contraction statement.}
The contraction factor in Section~\ref{sec:errcompare_main} is a sufficient-condition description of the ideal transition on the evaluated rollout support. Input-noised training motivates examining this property, but the training objective alone does not establish it, even at a population loss minimizer. We first give an explicit flow-field condition that implies contraction, then state the resulting error recursion. Neither the contraction coefficient nor the required derivative bounds have been measured in the reported experiments.

\paragraph{Rollout definition.}
Fix a context sequence $c_s$. Let $\mathcal{T}_\theta^{c_s}$ and $\mathcal{T}^{c_s}$ be the learned and ideal exact PF-ODE maps over $t\in[0,1-\epsilon]$. The comparison uses the same context in both transitions:
\begin{equation}
\label{eq:soft_rollout_app}
\mathbf{x}_{s+1}=\mathcal{T}_\theta^{c_s}(\mathbf{x}_s),\qquad
\mathbf{x}_{s+1}^*=\mathcal{T}^{c_s}(\mathbf{x}_s^*).
\end{equation}
Define $\delta_s=\mathbf{x}_s-\mathbf{x}_s^*$. This is a conditional comparison; the ideal sequence in this equation is not automatically the physical ground-truth trajectory with its own context.

\begin{lemma}[A sufficient condition for fixed-context contraction]
\label{lem:soft_contract_app}
Suppose Assumption~\ref{ass:idealstab_app} holds along the ideal paths from any two initial states in the relevant rollout domain, with the same context $c$. Then the ideal transition has Lipschitz constant at most
\[
\rho_\epsilon
=\exp\!\left(\int_0^{1-\epsilon}\frac{m(t)-1}{1-t}\,dt\right).
\]
In particular, for $L_T>0$ and $0\le k<1$, if
$\int_0^{1-\epsilon}(m(t)-1)/(1-t)\,dt\le\log\!\big(L_T(1-k)\big)$, then
\begin{equation}
\label{eq:soft_contract_app}
\|\mathcal{T}^{c}(\mathbf{x})-\mathcal{T}^{c}(\mathbf{y})\|
\le L_T(1-k)\|\mathbf{x}-\mathbf{y}\|.
\end{equation}
\end{lemma}
\begin{proof}
For the difference $d(t)$ of two ideal solutions, Assumption~\ref{ass:idealstab_app} gives
$\frac{1}{2}\frac{d}{dt}\|d(t)\|^2\le[(m(t)-1)/(1-t)]\|d(t)\|^2$.
Gronwall's inequality yields
$\|d(1-\epsilon)\|\le\rho_\epsilon\|d(0)\|$.
The additional integral condition gives Eq.~\eqref{eq:soft_contract_app}.
\end{proof}

\paragraph{Interpretation of the noise factor.}
For a hypothetical shared-noise preprocessing map $N_{k,z}(x)=(1-k)x+kz$ and an $L_T$-Lipschitz map $F$, the composition $F\circ N_{k,z}$ is $L_T(1-k)$-Lipschitz. This identity concerns an explicitly noised input. Our inference procedure applies no such preprocessing, so that identity is not a proof of Eq.~\eqref{eq:soft_contract_app} for the deployed transition. Here $L_T(1-k)$ parameterizes a proposed effective bound; its validity requires a separate condition such as Lemma~\ref{lem:soft_contract_app}. The parameter $L_T$ need not equal the Lipschitz constant of the physical dynamics, and its dependence on the trained model must not be ignored.

\paragraph{What the Gaussian regression example establishes.}
For independent scalar $X,Z\sim\mathcal{N}(0,1)$, fixed $k$, and $Y=(1-k)X+kZ$, the minimum mean-square-error predictor of $X$ is
\begin{equation}
\label{eq:bayes_app}
\mathbb{E}[X\mid Y=y]
=\frac{1-k}{(1-k)^2+k^2}\,y.
\end{equation}
Its sensitivity to $y$ is the full coefficient in this expression, not $1-k$. At $k=1/2$ the coefficient is $1$, and it exceeds $1$ for $0<k<1/2$. Thus denoising regression does not by itself guarantee contraction in the observed-input coordinate. This elementary regression example is also distinct from integrating a learned velocity field; it cannot identify the PF-ODE endpoint map without further assumptions.

\begin{assumption}[Effective Lipschitz bound on the rollout domain]
\label{ass:idealmaplip_app}
Uniformly over the contexts and initial states in the comparison, the ideal transition satisfies
$\|\mathcal{T}^{c}(x)-\mathcal{T}^{c}(y)\|\le\rho\|x-y\|$ for a finite $\rho\ge0$.
A sufficient flow-field condition is given in Lemma~\ref{lem:soft_contract_app}. The notation $\rho=L_T(1-k)$ used in the main text is valid only when Eq.~\eqref{eq:soft_contract_app} is independently satisfied.
\end{assumption}

\begin{theorem}[Conditional multi-step error recursion]
\label{thm:multistep_app}
Assume \ref{ass:scaledlip_app}, \ref{ass:idealstab_app}, and \ref{ass:idealmaplip_app} on all paths in the comparison. Let $B_s$ be the one-step bound in Lemma~\ref{lem:rescale_app}, evaluated with common initial state $\mathbf{x}_s$ and context $c_s$. Then
\begin{equation}
\label{eq:multistep_rec_app}
\|\delta_{s+1}\|\le\rho\|\delta_s\|+B_s,
\end{equation}
and, for $n\ge1$,
\begin{equation}
\label{eq:multistep_bound_app}
\|\delta_n\|\le\rho^n\|\delta_0\|
+\sum_{j=0}^{n-1}\rho^{n-1-j}B_j.
\end{equation}
If $\rho<1$ and $B_s\le\overline B$ uniformly, then
$\limsup_{n\to\infty}\|\delta_n\|\le\overline B/(1-\rho)$.
\end{theorem}
\begin{proof}
Add and subtract $\mathcal{T}^{c_s}(\mathbf{x}_s)$:
\[
\delta_{s+1}
=\big(\mathcal{T}_\theta^{c_s}(\mathbf{x}_s)-\mathcal{T}^{c_s}(\mathbf{x}_s)\big)
+\big(\mathcal{T}^{c_s}(\mathbf{x}_s)-\mathcal{T}^{c_s}(\mathbf{x}_s^*)\big).
\]
The first term is bounded by $B_s$ and the second by $\rho\|\delta_s\|$. Iteration proves Eq.~\eqref{eq:multistep_bound_app}; the uniform bound follows from the geometric series. The same inequality holds in expectation when its terms are integrable.
\end{proof}

\paragraph{Reconstructed context and numerical integration.}
For a physical reference sequence with its own context $c_s^*$, let
$b_s=\|\mathcal{T}^{c_s^*}(\mathbf{x}_s^*)-\mathbf{x}_{s+1}^*\|$ account for mismatch between the ideal finite-time flow map and the reference transition. Let $d_s$ bound the difference between the numerical solver and the exact learned PF-ODE map at step $s$. If the ideal transition is $L_c$-Lipschitz in context, the corresponding bound is
\[
\|\delta_{s+1}\|\le\rho\|\delta_s\|+B_s
+L_c\|c_s-c_s^*\|+b_s+d_s.
\]
This distinguishes velocity-model error, context error, reference mismatch, and numerical integration error. The exact-flow result does not automatically bound a finite-NFE implementation.

Context inheritance is restricted to the four-frame window. If its reconstruction map $C$ satisfies
$\|C(h_s)-C(h_s^*)\|\le L_C\max_{0\le j\le3}\|\delta_{s-j}\|$, where $h_s$ is the four-frame history, then
\[
\|\delta_{s+1}\|\le(\rho+L_cL_C)
\max_{0\le j\le3}\|\delta_{s-j}\|+B_s+b_s+d_s.
\]
Writing $q=\rho+L_cL_C$, a sufficient condition for a uniform bound is $q<1$ and $B_s+b_s+d_s\le\overline E$: induction gives
$\|\delta_s\|\le\max\{M_0,\overline E/(1-q)\}$, where $M_0$ bounds the initial four-frame errors. Rebuilding context limits the history but does not make its error contribution zero. These additional Lipschitz and error bounds are conditions, not reported measurements.

\paragraph{Relation to the training-noise sweep.}
Training samples $k\sim\mathcal{U}(0,k_{\max})$; the network does not receive $k$, and $k_{\max}=1$ does not mean that every training pair uses $k=1$. A result for one fixed noise level therefore does not imply a coefficient $L_T(1-k_{\max})$ for the trained mixture. In particular, it does not imply a zero-Lipschitz transition at the default $k_{\max}=1$. Table~\ref{tab:k_sweep} establishes an empirical decrease in 10-step L2RE over the tested noise ranges. It does not measure $\rho$, verify the flow-field condition, or prove that rollout error is monotone for other noise ranges or systems.

\paragraph{Comparison with deterministic rollouts.}
Both analyses are upper bounds under their respective assumptions. Proposition~\ref{prop:det_bound_app} allows geometric amplification when $L>1$ but does not require actual error growth. The conditional generative bound is stable when its effective coefficient is below one and the additive errors are uniformly bounded. Neither property follows from labeling a predictor deterministic or generative. The experiments motivate the conditional stability interpretation without establishing the required derivative bounds.

\section{Dataset description}

\setcounter{figure}{0}
\renewcommand{\thefigure}{B\arabic{figure}}
\label{dataset}
All the data are compressed to float16 (half) precision to enable the Data Distributed Parallel training on a 4 H-100 GPU node. The ratio of train:val:test is 8:1:1. Datasets are sampled with equal probabilities according to the practice in DPOT~\citep{hao2024dpot}.

\paragraph{FNO-v} We upsampled original data from c1p64 to c3p128 (the 2nd and 3rd dimension are filled with zero). The dataset size is expanded from 11.1GB to 21GB. Trajectory count: FNO-v5 -- 15.4k, FNO-v4 -- 368k, FNO-v3 -- 184k.

\paragraph{PDEArena} For the PDEArena-NavierStokes(PA-NS) and PDEArena-NavierStokesCond(PA-NSC), the dataset size is compressed from 60GB to 25GB. For the PDEArena-ShallowWaterEquation(PA-SWE), its size was reduced from 76.6GB to 62GB. Trajectory count: PA-NS -- 48k, PA-NSC -- 120k, PA-SWE -- 470k.

\paragraph{PDEBench} For the PDEBench-CompressibleNavierStokes(PB-CNS), unimportant physical fields are filtered. Thus, it becomes 65GB compressed from 551GB. For the PDEBench-ShallowWaterEquation(PB-SWE), it is compressed to 0.3GB from 6.2GB, the 2nd and 3rd dimension is filled with zero. Trajectory count: PB-CNS -- 598k, PB-SWE -- 77.6k.

\paragraph{The Well} For the Well-GrayScott(W-GS), we fill the 3rd dimension with zero, ending up with a 5.3GB data set compressed from 153GB. For the Well-ActiveMatter(W-AM), we downsampled the data from c3p256 to c3p128, and obtained a compressed 1.1GB dataset from 51.3GB. For the Well-PlanetShallowWaterEquation(W-SWE), we downsampled the data from c3p256,512 to c3p128 and filtered out unimportant fields, so the data size is compressed to 9.3GB from 185.8GB. For the Well-RayleighBenard(W-RB), we downsampled the data from c3p512,128 to c3p128, and get a 26GB dataset from 342GB original data. For the Well-ShearFlow(W-SF), it is compressed to 14GB from 547GB by filtering out unimportant fields. For the Well-TurbulentRadiativeLayer2D(W-TR), it is downsampled from c3p128,384 to c3p128, thus compressed to 0.5GB from 6.9GB. For the Well-ViscoElasticInstability(W-VE), it is downsampled from c3p512 to c3p128, thus compressed to 0.5GB from 66GB. Trajectory count: W-GS -- 92.2k, W-AM -- 13.4k, W-SWE -- 96.4k, W-RB -- 266.6k, W-SF -- 175.6k, W-TR -- 7k, W-VE -- 5.3k.

\begin{table}[h]
\caption{Channel mapping for the unified c3p128 format.  Channels
marked with * are zero-padded.}
\label{tab:channel_map}
\centering
\small
\begin{tabular}{llll}
\toprule
System & Original & Retained channels & Padding \\
\midrule
FNO-v5   & c1p64  & vorticity & ch.\,2--3 zero* \\
FNO-v4   & c1p64  & vorticity & ch.\,2--3 zero* \\
FNO-v3   & c1p64  & vorticity & ch.\,2--3 zero* \\
PA-NS    & c3p128 & $u_x, u_y, p$ & --- \\
PA-NSC   & c4p128 & $u_x, u_y, p$ (drop buoyancy) & --- \\
PA-SWE   & c3p128 & $h, u_x, u_y$ & --- \\
PB-CNSL  & c4p128 & $\rho, u_x, p$ (drop $u_y$) & --- \\
PB-CNSH  & c4p128 & $\rho, u_x, p$ (drop $u_y$) & --- \\
PB-SWE   & 1 channel & scalar field & ch.\,2--3 zero* \\
W-AM     & c3p256 & all three source channels & --- \\
W-GS     & c2p64  & $u, v$ & ch.\,3 zero* \\
W-RB     & c2p128 & $T, \omega$ & ch.\,3 zero* \\
W-SWE    & c3p128 & $h, u_x, u_y$ & --- \\
W-SF     & c1p128 & stream function & ch.\,2--3 zero* \\
W-TR     & c1p128 & tracer & ch.\,2--3 zero* \\
W-VE     & c3p128 & $\sigma_{xx}, \sigma_{xy}, \sigma_{yy}$ & --- \\
\bottomrule
\end{tabular}
\end{table}

\section{Model Architectures Details} \label{apx:imp}

\subsection{Latent Generative Solvers} \label{apx:imp:lgs}

PhyVAE is optimized with AdamW using $\beta_1=0.9$, $\beta_2=0.995$, a cosine learning-rate schedule with 10\% linear warmup, and weight decay $10^{-4}$. PFlowFT uses AdamW with $\beta_1=0.9$, $\beta_2=0.95$, the same cosine schedule and warmup ratio, and weight decay $10^{-4}$. Base learning rates are set to $10^{-4}$ at batch size 256 and scaled linearly with batch size to balance convergence speed and stability.

Each FlowFT contains 12 SiT~\citep{ma2024sitexploringflowdiffusionbased} blocks with embedding dimension 768 and 12 attention heads, corresponding to one head per 64 model dimensions. LGS has 138M predictor parameters and costs 1.334 GFLOPs per sample, excluding latent encoding, which is precomputed.

\subsection{Baselines} \label{apx:imp:baseline}
All baselines are trained for 200k steps with the same optimizer configuration used for LGS.
CNextUNet uses internal dimension 192 with standard ConvNeXt blocks and U-Net skip connections, giving 153M parameters and 17.61 GFLOPs per sample. U-AFNO uses the same U-Net scaffold, with 16 AFNO blocks, embedding dimension 1024, and 64 retained Fourier modes at the lowest resolution. It has 152M parameters and costs 90.305 GFLOPs per sample. DPOT uses patch embedding with patch size 8 and transposed-convolution upsampling to reduce the internal token count; Fourier channel mixing with 64 retained modes replaces self-attention. We stack 16 DPOT blocks with embedding dimension 1024, giving 160M parameters and 103.224 GFLOPs per sample.

\subsection{Ablation Studies} \label{apx:imp:ablation}
Below we restate each ablation using the conditional objectives \(p(\mathbf{x}_{s+1}\mid \cdot)\) consistent with Sec.~\ref{methods}.

\paragraph{A1: No pyramids.}
Same modeling objective as the full method, but conditioning uses the full-resolution history (no temporal downsampling), increasing attention cost.

\begin{equation}
\label{eq:ablate_A1_obj}
p(\mathbf{x}_{s+1}\mid \tilde{\mathbf{x}}_s,\ t,\ c)p(c\mid  \hat{\mathbf{x}}_{1:s})
\end{equation}
where $\tilde{\mathbf{x}}_s$ is a $k$-noised $\mathbf{x}_s$ at intermediate transport time $t$.

\paragraph{A2: No per-trajectory context maintenance.}
We remove the latent dynamics variable \(c_s\) and condition directly on noisy history tokens and time marker:
\begin{equation}
\label{eq:ablate_A2_obj}
p(\mathbf{x}_{s+1}\mid \tilde{\mathbf{x}}_s,\ \hat{\mathbf{x}}_{1:s},\ t),
\end{equation}
where $\hat{\mathbf{x}}_{1:s}$ are generated from $k$-noised $\mathbf{x}_{0:s-1}$ in Eq.~\eqref{eq:softsource_main}.

\paragraph{A3: No noise level \(k\).}
We set \(k=0\) in Eq.~\eqref{eq:softsource_main} but the rest remains the same as ~\cref{eq:ablate_A2_obj}.

\paragraph{A4: No generative transition.}
We replace generative prediction with direct regression:
\begin{equation}
\label{eq:ablate_A4_obj}
\hat{\mathbf{x}}_{s+1}=f(\mathbf{x}_{0:s}),\qquad
\min_\theta\ \mathbb{E}\big[\|\hat{\mathbf{x}}_{s+1}-\mathbf{x}_{s+1}\|_2^2\big].
\end{equation}

To ablate temporal pyramids, we remove the state-interpolation step before attention computation. This increases GFLOPs by approximately $4\times$ and provides only marginal accuracy gains. To remove physics-context maintenance, we concatenate predicted states and apply cross-attention to current-state embeddings, yielding a Perceiver-style variant~\citep{jaegle2021perceivergeneralperceptioniterative}. To ablate the noise level $k$, we set $k=0$ during intermediate-state synthesis in training. To reduce the method to a non-generative predictor, we keep the Perceiver-style conditioning but remove diffusion-time sampling and directly regress the next state from the untransported current state, corresponding to one-step $x$-prediction supervision.

\paragraph{Single-component (isolated) ablations.}
The cumulative chain A1$\to$A4 conflates contributions, so we additionally run three single-component variants that remove exactly one ingredient: \textbf{A2S} drops the context $c_s$, \textbf{A3S} sets $k\!=\!0$, and \textbf{A4S} replaces flow matching with direct $\ell_2$ regression. Numbers are in Table~\ref{tab:ablation_single}; the qualitative picture is that each removal answers a different question. A4S asks whether the gain comes from the generative design itself ($18.3\%\!\to\!46.5\%$ at 10-step, $2.5\times$ worse---yes). A3S asks whether input noising is essential to long-horizon stability ($23.5\%$ at 10-step---also yes). A2S asks whether shared conditioning across heterogeneous systems is essential ($24.5\%$---yes again). A1 (no pyramids) marginally improves accuracy at ${\sim}4\times$ FLOPs, confirming pyramids buy efficiency, not accuracy.

\section{Visualization for test sets of 15 PDE systems}

\setcounter{figure}{0}
\renewcommand{\thefigure}{E\arabic{figure}}
\label{visualization}
FNO-v: sampled trajectories are displayed in Fig.~\ref{fig:fno}.
\begin{figure}[ht]
    \centering
    \includegraphics[width=0.5\linewidth]{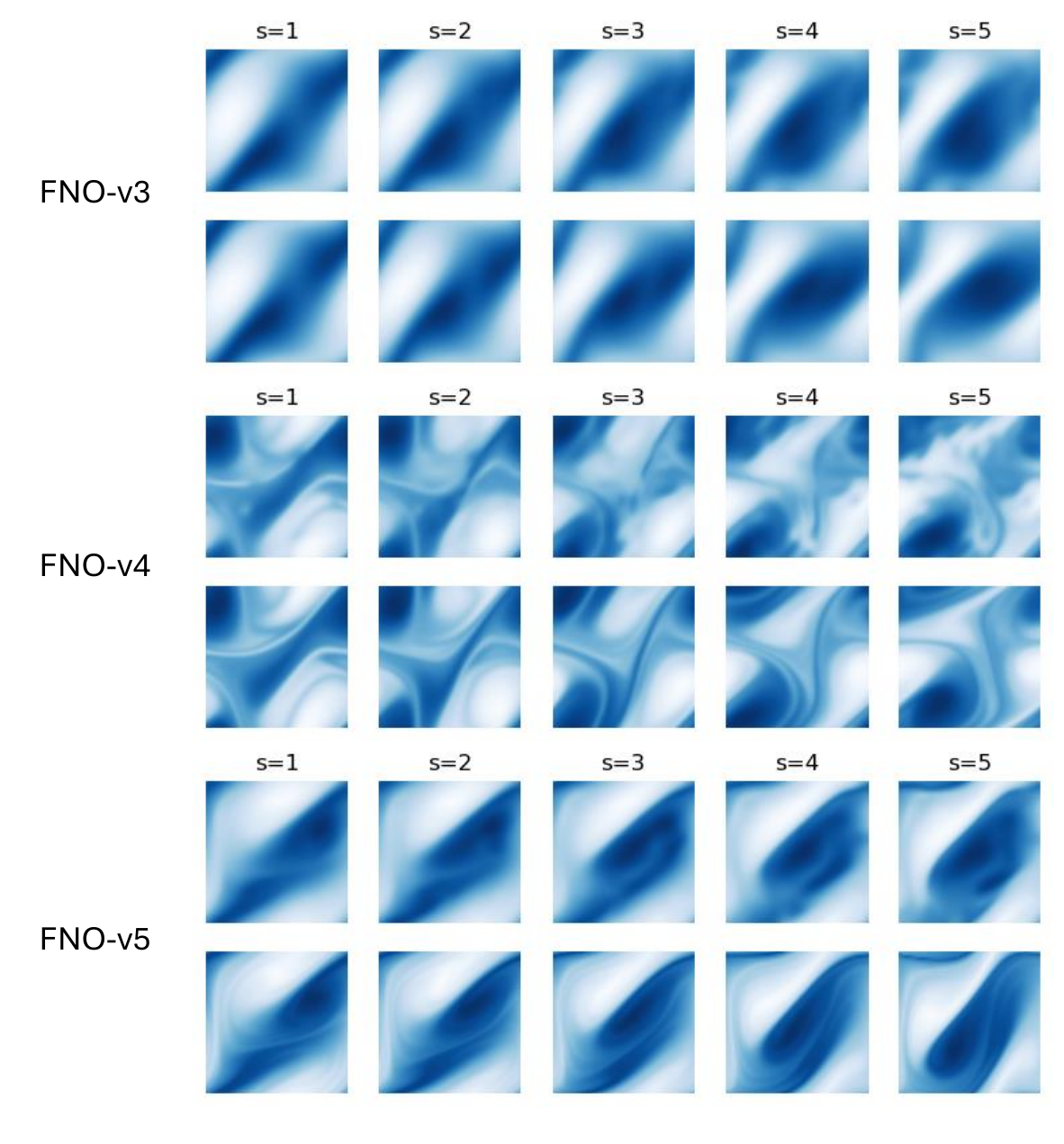}
    \caption{Sampled trajectories from FNO-v3, FNO-v4, FNO-v5. Upper row: prediction. Bottom row: ground truth.}
    \label{fig:fno}
\end{figure}

PDEArena: sampled trajectories are displayed in Fig.~\ref{fig:pa}.
\begin{figure}[ht]
    \centering
    \includegraphics[width=0.5\linewidth]{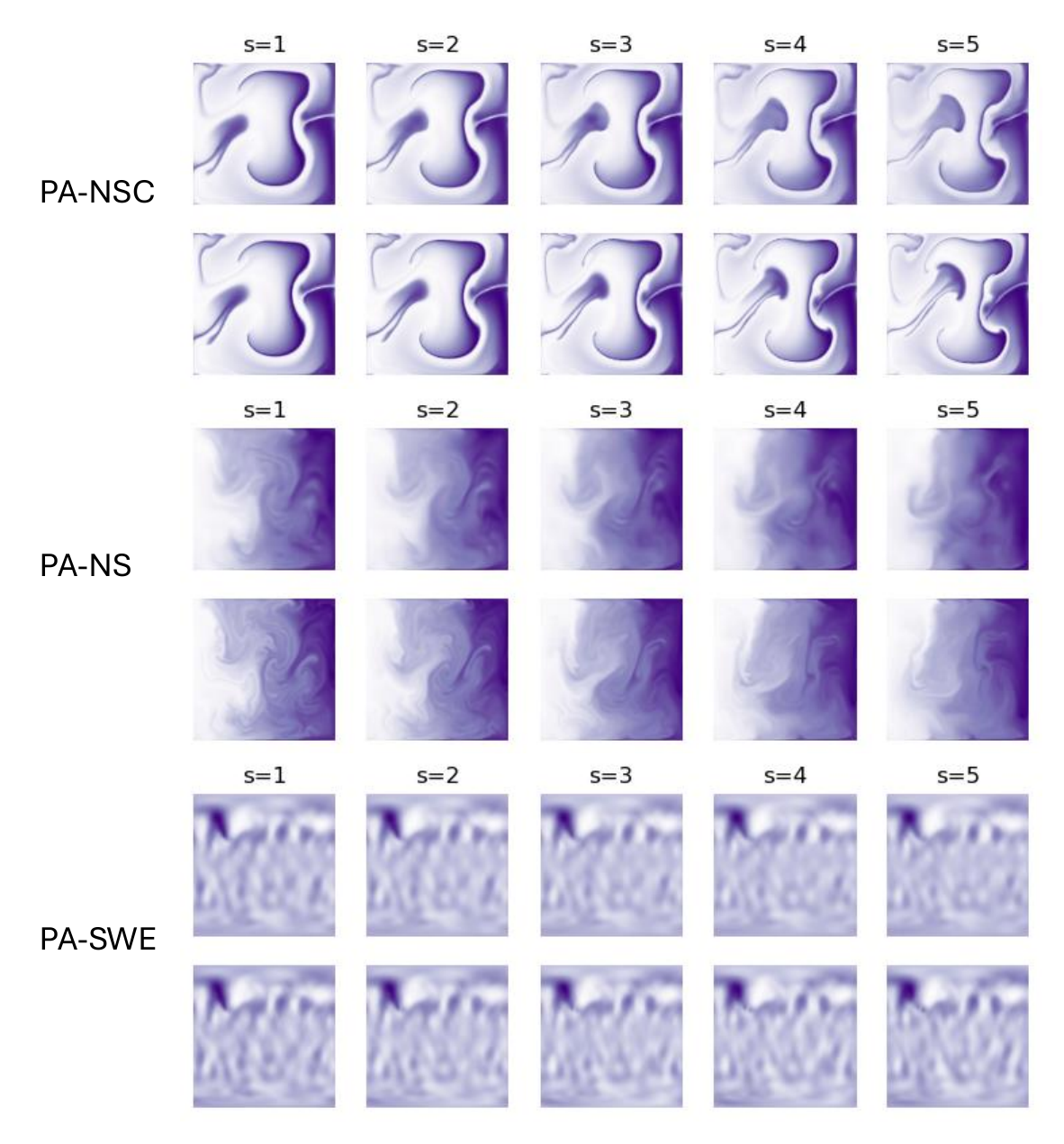}
    \caption{Sampled trajectories from PA-NSC, PA-NS, PA-SWE. Upper row: prediction. Bottom row: ground truth.}
    \label{fig:pa}
\end{figure}

PDEBench: sampled trajectories are displayed in Fig.~\ref{fig:pb}.
\begin{figure}[ht]
    \centering
    \includegraphics[width=0.5\linewidth]{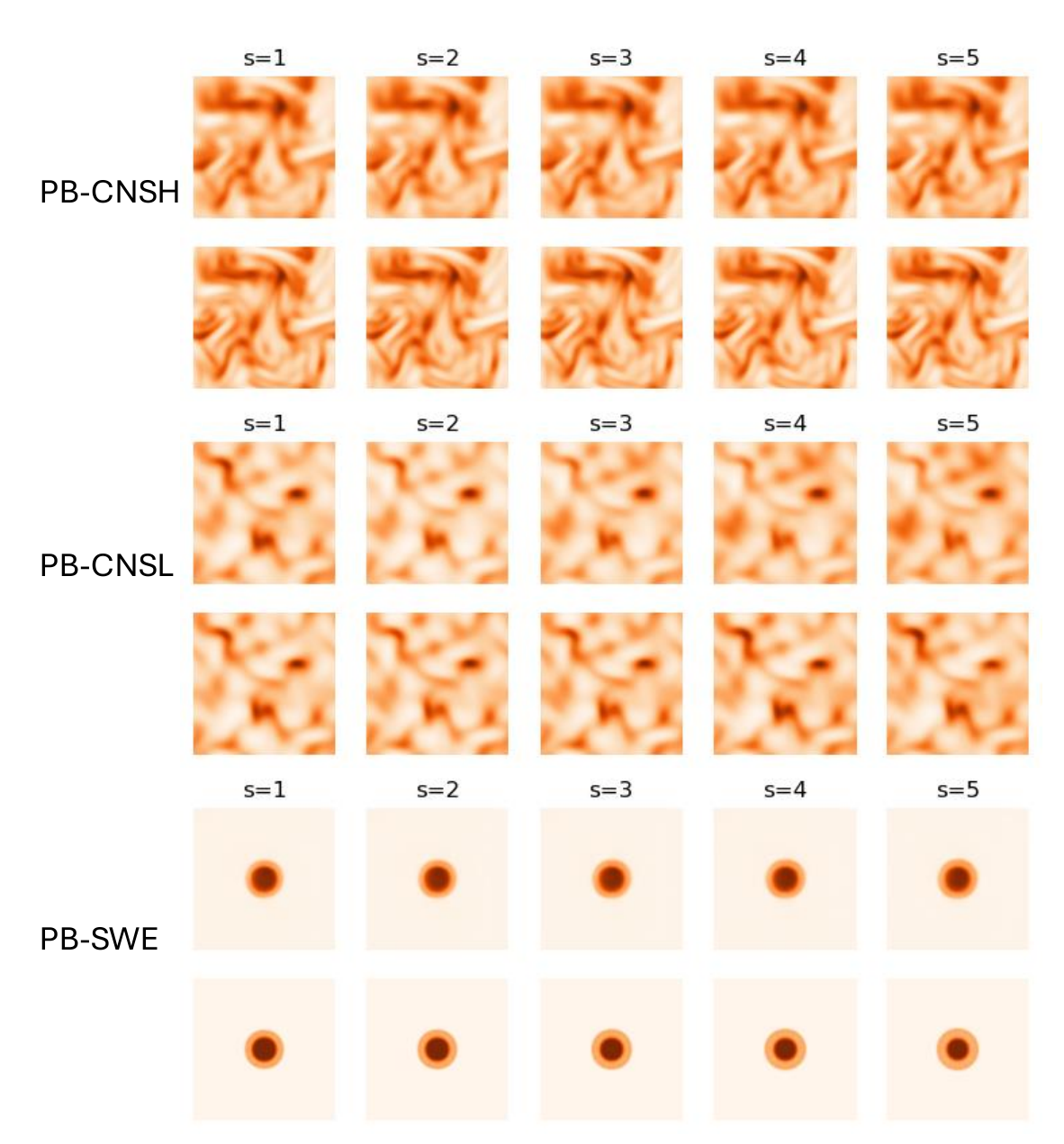}
    \caption{Sampled trajectories from PB-CNSH, PB-CNSL, PB-SWE. Upper row: prediction. Bottom row: ground truth.}
    \label{fig:pb}
\end{figure}

The Well: sampled trajectories are displayed in Fig.~\ref{fig:well}.
\begin{figure}[ht]
    \centering
    \includegraphics[width=\linewidth]{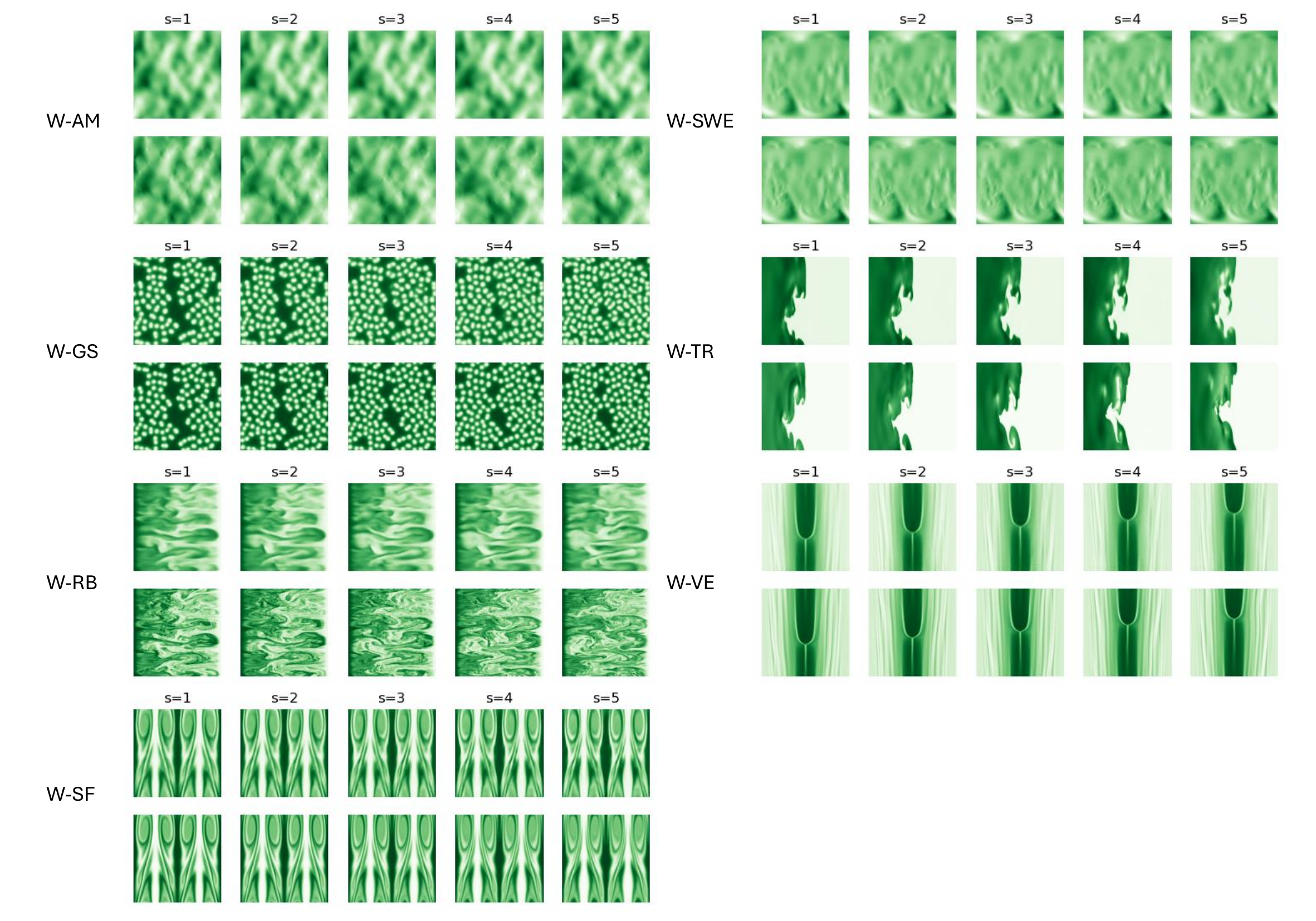}
    \caption{Sampled trajectories from W-AM, W-GS, W-RB, W-SF, W-SWE, W-TR, W-VE. Upper row: prediction. Bottom row: ground truth.}
    \label{fig:well}
\end{figure}

\section{Rollout visualizations}

\setcounter{figure}{0}
\renewcommand{\thefigure}{G\arabic{figure}}
\label{rollout}
\rev{A matched long-term PDEArena Navier--Stokes trajectory is provided in Fig.~\ref{fig:pans}.}
\begin{figure}[ht]
    \centering
    \includegraphics[width=1\linewidth]{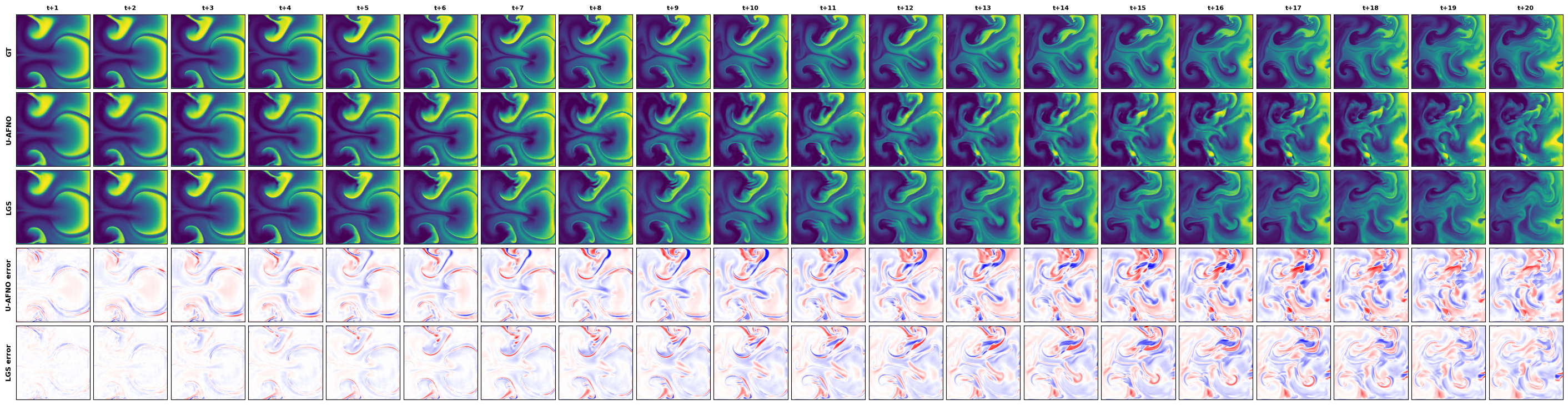}
    \caption{\rev{Matched 20-step trajectories for PDEArena Navier--Stokes: ground truth, LGS-138M, U-AFNO-152M, and their errors. Physical-field color limits are identical across methods at each horizon.}}
    \label{fig:pans}
\end{figure}

\section{Additional Experiments}
\label{apx:tables}

\setcounter{table}{0}
\renewcommand{\thetable}{H\arabic{table}}
\setcounter{figure}{0}
\renewcommand{\thefigure}{H\arabic{figure}}

\rev{\subsection{Additional Evaluation and Protocol Details}}

\rev{\paragraph{Scope and interpretation.}
The empirical claims concern the evaluated 2D regular-grid systems and discrete horizons. The shared latent tensor is an engineering interface for pooled reconstruction and prediction, not evidence that all PDEs share one semantic state space. Boundary behavior is learned from the training fields rather than analytically enforced; coupled physics, irregular or adaptive domains, and 3D fields remain untested.}

\rev{\paragraph{Related regular-grid transformer.}
PDE-Transformer~\citep{holzschuh2025pdetransformer} is a close deterministic, multi-scale regular-grid surrogate. LGS instead studies a compressed shared interface with a generative transition and input noising. Because the datasets and training protocols differ, we compare these designs conceptually rather than report an unmatched ranking.}

\rev{\paragraph{Physical-time mapping.}
One model step is one adjacent processed frame. Table~\ref{tab:physical_time} records the stored-frame cadence from the public generation configurations or time coordinates of FNO-v~\citep{li2020fourier}, PDEArena~\citep{gupta2022multispatiotemporalscalegeneralizedpdemodeling}, PDEBench~\citep{takamoto2024pdebenchextensivebenchmarkscientific}, and The Well~\citep{ohana2025welllargescalecollectiondiverse}. ``Sim.\ unit'' denotes a provider coordinate without an assigned SI unit.}

\begin{table}[ht]
\centering
\small
\caption{\rev{Stored-frame cadence and physical span. A dash indicates insufficient future frames.}}
\label{tab:physical_time}
\begin{tabular}{lrrrr}
\toprule
System & Cadence & 5 steps & 10 steps & 20 steps \\
\midrule
FNO-v3 & 1 sim.\ unit & 5 & 10 & 20 \\
FNO-v4 & 0.6 sim.\ unit & 3 & 6 & 12 \\
FNO-v5 & 1 sim.\ unit & 5 & 10 & -- \\
PA-NS & 1.5 sim.\ units & 7.5 & 15 & 30 \\
PA-NSC & 6 sim.\ units & 30 & 60 & -- \\
PA-SWE & 5/6 h & 27--28 h & 54--55 h & 109--110 h \\
PB-CNSL/H & 0.05 sim.\ unit & 0.25 & 0.5 & -- \\
PB-SWE & 0.01 s & 0.05 s & 0.10 s & 0.20 s \\
W-AM & 0.25 s & 1.25 s & 2.5 s & 5 s \\
W-GS & 100 s & 500 s & 1,000 s & 2,000 s \\
W-SWE & 1 h & 5 h & 10 h & 20 h \\
W-RB & 0.25 sim.\ unit & 1.25 & 2.5 & 5 \\
W-SF & 0.1 sim.\ unit & 0.5 & 1 & 2 \\
W-TR & 1.597033 sim.\ units & 7.985165 & 15.97033 & 31.94066 \\
W-VE & 1/4 sim.\ units & 5/20 & 10/40 & -- \\
\bottomrule
\end{tabular}
\end{table}
\FloatBarrier

\rev{PA-SWE has irregular 5/6-hour gaps, W-VE cadence depends on trajectory, and preprocessing retains every tenth public W-GS snapshot. Forty PA-NS steps span 60 provider simulation units.}

\rev{\paragraph{Matched-budget pyramid control.}
Table~\ref{tab:pyramid_control} compares the submitted spatial allocation with the full-resolution A1 ablation and a matched-budget control. Data, architecture, 200\,k predictor-update budget, input noising, context, and evaluation are unchanged.}

\begin{table}[ht]
\centering
\small
\caption{\rev{Fixed four-frame history controls (L2RE in \%).}}
\label{tab:pyramid_control}
\begin{tabular}{lrrrrr}
\toprule
Allocation & Tokens & 1 step & 5 steps & 10 steps & Throughput \\
\midrule
$2/4/8/16$ (LGS) & 340 & 6.0 & 12.1 & \textbf{18.3} & 75.72/s \\
$16/16/16/16$ (A1) & 1,024 & \textbf{5.4} & \textbf{11.3} & \textbf{17.6} & -- \\
$5/5/5/16$ & 331 & 10.2 & 14.7 & 21.2 & 55.20/s \\
\bottomrule
\end{tabular}
\end{table}
\FloatBarrier

\rev{The submitted and 331-token controls have the same 5.101-GiB peak memory. Full-resolution history is slightly more accurate but requires approximately four times the predictor FLOPs; the matched-budget result supports progressive allocation rather than token count alone.}

\begin{figure}[ht]
\centering
\small
\[
\underbrace{
\mathbf{x}_{s-3}^{\,2\times2}\ \longrightarrow\
\mathbf{x}_{s-2}^{\,4\times4}\ \longrightarrow\
\mathbf{x}_{s-1}^{\,8\times8}\ \longrightarrow\
\mathbf{x}_{s}^{\,16\times16}}_{\text{340 history tokens}}
\ \xrightarrow{\ \mathrm{PFlowFT}\ }\ \hat{\mathbf{x}}_{s+1}
\]
\[
\mathbf{x}_{s-2:s+1}\quad\xrightarrow{\text{shift window; rebuild context}}\quad c_{s+1}.
\]
\caption{\rev{Fixed-window spatial pyramid. Older frames receive fewer spatial tokens; after prediction, context is discarded and reconstructed from the shifted four-frame window.}}
\label{fig:pyramid_schematic}
\end{figure}
\FloatBarrier

\paragraph{Context inheritance within the window.}
The context updates in Eq.~\eqref{eq:condupdate} operate within the current four-frame window. After each window shift, context is reconstructed from the retained frames and the new prediction. Context inheritance is therefore restricted to the four-frame history represented by the pyramid.

\rev{This allocation follows the older-history compression principle of Pyramidal Flow Matching~\citep{jin2025pyramidalflowmatchingefficient}. The experiments compare one-, two-, and four-resolution allocations, not a clean level-count sweep. Temporal stride is one stored frame, pooling is average pooling, and context length is four. Other strides, interpolation rules, context lengths, learned memory tokens, recurrent summaries, and hierarchical attention remain untested.}

\rev{\paragraph{Forty-step rollout.}
Using the same checkpoints, aggregation, and 11-system population as the 20-step comparison, LGS averages 37.3\% versus 76.1\% for U-AFNO and is lower on all 11 systems (Table~\ref{tab:forty_step}).}

\begin{table}[ht]
\centering
\small
\setlength{\tabcolsep}{3pt}
\caption{\rev{Per-system 40-step L2RE (\%).}}
\label{tab:forty_step}
\begin{tabular}{lrr|lrr}
\toprule
System & LGS & U-AFNO & System & LGS & U-AFNO \\
\midrule
FNO-v3 & \textbf{50.8} & 85.0 & W-GS  & \textbf{21.8} & 58.6 \\
FNO-v4 & \textbf{51.4} & 129.0 & W-RB & \textbf{35.8} & 71.2 \\
PA-NS  & \textbf{35.6} & 54.0 & W-SF  & \textbf{11.0} & 73.8 \\
PA-SWE & \textbf{71.0} & 101.6 & W-SWE & \textbf{19.6} & 48.1 \\
PB-SWE & \textbf{26.3} & 115.4 & W-TR & \textbf{28.2} & 28.8 \\
W-AM   & \textbf{58.4} & 71.7 & \textbf{Average} & \textbf{37.3} & 76.1 \\
\bottomrule
\end{tabular}
\end{table}
\FloatBarrier

\rev{\paragraph{Spectral and temporal diagnostics.}
Log spectral distance (LSD) and absolute lag-1 temporal-autocorrelation error use identical denormalization for both models and exclude zero-padded channels. The 5/10-step values average over 16 systems and the 20-step values over the 11 eligible systems (Table~\ref{tab:diagnostics}).}

\begin{table}[ht]
\centering
\small
\caption{\rev{Diagnostics beyond L2RE (lower is better).}}
\label{tab:diagnostics}
\begin{tabular}{lrrrr}
\toprule
Metric & Model & 5 steps & 10 steps & 20 steps \\
\midrule
\multirow{2}{*}{LSD}
& LGS & \textbf{1.981533} & \textbf{2.188857} & 2.294261 \\
& U-AFNO & 2.202804 & 2.375815 & \textbf{2.260238} \\
\multirow{2}{*}{Lag-1 error}
& LGS & \textbf{0.00402120} & \textbf{0.00416493} & \textbf{0.00463262} \\
& U-AFNO & 0.00414658 & 0.00503028 & 0.00665576 \\
\bottomrule
\end{tabular}
\end{table}
\FloatBarrier

\rev{LGS has lower lag-1 error at every horizon and lower LSD at 5 and 10 steps; U-AFNO is lower by 0.034 at 20-step LSD. These diagnostics test smoothing and short-lag drift.}

\rev{\paragraph{PA-NS specialist control.}
Both specialists use the same PA-NS split, 200,000 updates at global batch 64, the identical 7,904-window test set, and a final-update checkpoint rule fixed before 40-step evaluation. Specialist LGS uses cached latents from the submitted frozen PhyVAE, isolating temporal-model specialization (Table~\ref{tab:specialist}).}

\begin{table}[ht]
\centering
\small
\caption{\rev{PA-NS pooled and specialist L2RE (\%).}}
\label{tab:specialist}
\begin{tabular}{llrrrrr}
\toprule
Model & Training & 1 & 5 & 10 & 20 & 40 steps \\
\midrule
LGS & Pooled & 16.6 & 23.1 & 27.5 & 31.7 & 35.6 \\
LGS & Specialist & 8.9 & 18.4 & \textbf{23.4} & \textbf{28.8} & \textbf{32.4} \\
U-AFNO & Pooled & 16.4 & 31.3 & 30.3 & 41.6 & 54.0 \\
U-AFNO & Specialist & \textbf{6.4} & \textbf{18.3} & 30.4 & 36.5 & 37.9 \\
\bottomrule
\end{tabular}
\end{table}
\FloatBarrier

\rev{Specialization improves both models at short horizons. Specialist LGS is lower at 10/20/40 steps; this result is limited to PA-NS and does not compare shared with family-specific codecs.}

\rev{\paragraph{Hyperparameter selection.}
We use one global four-frame configuration rather than tune each PDE family separately, preserving the pooled-training comparison and aligned update budgets. The input-noise bound and latent channel count are ablated directly; the remaining architecture and optimizer settings are fixed globally rather than selected by an exhaustive search.}

\rev{\paragraph{Probabilistic diagnostics.}
CRPS requires a validated predictive-ensemble protocol, while ACC requires family-specific climatologies; both remain separate evaluation settings.}

\FloatBarrier

\paragraph{Computational profiling.}
The profiling below is measured on a single H200 GPU with batch size 64 using \texttt{torch.compile} (\cref{tab:profiling}). PhyVAE encoding is precomputed and cached as a one-time cost (${\sim}70$\,GFLOPs/frame, completing in under 10 GPU-minutes for the full corpus); each predictor evaluation uses the 138M predictor at 1.3\,GFLOPs/sample. PhyVAE decoding (113\,GFLOPs) is required whenever latent predictions are converted to physical fields for evaluation or deployment.
\rev{The reported LGS row is one predictor NFE. The submitted integrator uses 20 NFEs per physical transition, approximately 26 predictor GFLOPs and 48\,ms before decoding; a decoded transition is approximately 139 GFLOPs. Thus, the reported compute advantage concerns training and is not an inference-latency claim. NFE count, solver order, and adaptive integration were not swept; decoded wall-clock time was not measured.}

\begin{table}[h]
\caption{Computational profiling on a single H200 GPU (batch 64).
PhyVAE encoding is precomputed and amortized;
\rev{the LGS row is one of 20 predictor NFEs per physical transition.}}
\label{tab:profiling}
\centering
\small
\begin{tabular}{lcccc}
\toprule
\textsc{Model} & Params & GFLOPs & GPU RAM & Wall-clock \\
\midrule
CNextUNet & 153M & 17.6  & 78.6\,GB & 11.6\,ms \\
DPOT       & 160M & 103.2 & 54.6\,GB &  4.5\,ms \\
U-AFNO     & 152M & 90.3  & 72.7\,GB &  5.7\,ms \\
LGS (predictor) & 138M & 1.3 & 33.1\,GB & 2.4\,ms \\
\bottomrule
\end{tabular}
\end{table}

\paragraph{Training-time $k$ sensitivity.}
We train five separate models, each at a different upper bound $k_{\max} \in \{0, 0.3, 0.5, 0.7, 1.0\}$ on the per-step noise-level sampling distribution (recall that $k\!\sim\!\mathcal{U}(0,k_{\max})$ at every training step; $k_{\max}\!=\!0$ recovers the deterministic-input training setup, $k_{\max}\!=\!1$ gives the model the widest exposure from clean to fully-noised inputs but does \emph{not} mean every step uses pure noise). The full architecture is otherwise intact (temporal pyramids, per-trajectory context, generative transition), cleanly isolating the effect of the noise-range bound. All $k_{\max}$ values produce comparable 1-step accuracy (6.0--6.5\%), but a wider training noise range consistently improves long-horizon stability: 10-step L2RE decreases from 23.5\% ($k_{\max}{=}0$) to 18.3\% ($k_{\max}{=}1$), a 22\% relative improvement, supporting the empirical benefit of the tested training-noise ranges. The sweep does not measure the contraction coefficient assumed in Theorem~\ref{thm:multistep_app}.

\begin{table}[h]
\caption{Training-time noise-range sweep. Each row is a separately trained model with a different upper bound $k_{\max}$ on the per-step noise-level sampling distribution: at every training step $k\!\sim\!\mathcal{U}(0,k_{\max})$, so $k_{\max}\!=\!1$ exposes the model to the full range from clean to fully-noised inputs (not pure noise on every step). \textbf{Inference applies no input noising} in any row---only the autoregressive output is fed to the PF-ODE solver. All variants share the full architecture and training schedule. $k_{\max}{=}0$ corresponds to ablation A3S (deterministic-input training); $k_{\max}{=}1.0$ is the LGS default.}
\label{tab:k_sweep}
\centering
\small
\begin{tabular}{lccc}
\toprule
$k_{\max}$ & 1-step & 5-step & 10-step \\
\midrule
0.0 (A3S)         & 6.5 & 15.3 & 23.5 \\
0.3               & 6.3 & 14.6 & 22.3 \\
0.5               & 6.3 & 14.7 & 22.1 \\
0.7               & 6.0 & 13.7 & 21.6 \\
1.0 (LGS default) & 6.0 & 12.1 & 18.3 \\
\bottomrule
\end{tabular}
\end{table}

\paragraph{Single-component ablations.}
See Appendix~\ref{apx:imp:ablation} for full design details. Each variant removes exactly one component while keeping all others intact: A2S removes per-trajectory context $c_s$ only; A3S sets $k_{\max}{=}0$ (deterministic-input training, no per-step noise sampling); A4S replaces flow matching with direct regression only. A4S shows the largest degradation ($18.3\% \to 46.5\%$ at 10-step), confirming the generative transition as the most critical component. A3S and A2S both degrade substantially at long horizons while maintaining reasonable 1-step accuracy, showing $k$ and context contribute through complementary mechanisms.

\begin{table}[h]
\caption{Single-component ablation study.  Each row removes exactly one
component while keeping all others intact. Subscripts denote standard deviation across the 16 PDE systems.}
\label{tab:ablation_single}
\centering
\small
\begin{tabular}{llccc}
\toprule
Ablation & Removed & 1-step & 5-step & 10-step \\
\midrule
LGS (full) & ---              & $6.0_{\pm 3.6}$ & $12.1_{\pm 6.4}$ & $18.3_{\pm 10.1}$ \\
A1          & Temporal pyramids & $5.4_{\pm 3.6}$ & $11.3_{\pm 6.8}$ & $17.6_{\pm 10.3}$ \\
A2S         & Physics context   & $6.9_{\pm 3.6}$ & $16.3_{\pm 10.4}$ & $24.5_{\pm 15.4}$ \\
A3S         & Input noising ($k{=}0$) & $6.5_{\pm 3.4}$ & $15.3_{\pm 9.6}$ & $23.5_{\pm 14.2}$ \\
A4S         & Generative solver & $8.0_{\pm 5.1}$ & $29.6_{\pm 18.7}$ & $46.5_{\pm 26.4}$ \\
\bottomrule
\end{tabular}
\end{table}

\paragraph{Scaled baseline comparison (${\sim}$400M parameters).}
To address the concern that LGS's advantage stems from its larger total parameter count (249M PhyVAE + 138M predictor = 387M), we scale U-AFNO from 152M to 404M by doubling the embedding dimension to 2048, and scale DPOT from 160M to 396M similarly. All scaled models are trained with the same schedule and optimizer. Scaled baselines have lower 1-step error (4.2--4.3\% vs.\ LGS 6.0\%), but LGS still achieves the best 10-step accuracy (18.3\%) vs.\ U-AFNO-404M (23.9\%) and DPOT-396M (19.7\%), showing that these parameter increases do not eliminate the 10-step gap.

\begin{table}[h]
\caption{Scaled baseline comparison.  Baselines scaled to
${\sim}400$M parameters to match LGS's total count.}
\label{tab:scaled_baselines}
\centering
\small
\begin{tabular}{lcccc}
\toprule
Model & Params & 1-step & 5-step & 10-step \\
\midrule
U-AFNO-152M        & 152M & 5.0 & 17.0 & 29.6 \\
U-AFNO-404M        & 404M & 4.2 & 13.4 & 23.9 \\
DPOT-160M          & 160M & 6.1 & 16.9 & 29.6 \\
DPOT-396M          & 396M & 4.3 & 11.5 & 19.7 \\
LGS (138M active)  & 387M & 6.0 & 12.1 & 18.3 \\
\bottomrule
\end{tabular}
\end{table}

\paragraph{Generative baseline comparison.}
We adapt two recent generative PDE methods to our autoregressive benchmark. \textbf{DiffusionPDE}~\citep{huang2024diffusionpdegenerativepdesolvingpartial} was originally designed for inverse/reconstruction problems; we adapt it by concatenating 4 previous frames as UNet conditioning (12+3 input channels) using the EDM framework with Heun's 2nd-order sampler (50 steps). \textbf{FunDiff}~\citep{wang2025fundiffdiffusionmodelsfunction} operates in function space; we adapt it by encoding frames via its Function AutoEncoder into latent tokens, fusing 4-frame history via projection, and training a rectified flow in FAE latent space with Euler ODE (100 steps). LGS's latent-space flow matching with input noising provides clear advantages over both alternative generative strategies.

\paragraph{PhyVAE latent channel ablation.}
We ablate the PhyVAE latent channel count following Hansen-Estruch et al.\ (2025), keeping the spatial resolution fixed at $16{\times}16$ and retraining the dynamics predictor from scratch on each variant's latent trajectories. The default c16 achieves the best 10-step accuracy (18.3\%). Both c8 (higher compression) and c32 (lower compression) degrade, suggesting too few channels lose dynamical information while too many introduce redundancy that hinders temporal learning.

\begin{table}[h]
\caption{PhyVAE latent channel ablation.  For each variant, the dynamics
predictor is retrained from scratch on the corresponding latent trajectories.}
\label{tab:vae_ablation}
\centering
\small
\begin{tabular}{lcccc}
\toprule
Variant & Latent Shape & Recon.\ L2RE & 1-step & 10-step \\
\midrule
PhyVAE-c8            & c8p16  & 3.7 & 6.8 & 22.0 \\
PhyVAE-c16 (default) & c16p16 & 3.01 & 6.0 & 18.3 \\
PhyVAE-c32           & c32p16 & 3.6 & 6.1 & 20.5 \\
\bottomrule
\end{tabular}
\end{table}

\paragraph{Robustness to observation noise.}
We inject per-frame i.i.d.\ Gaussian noise $\boldsymbol{\epsilon} \sim \mathcal{N}(0, \sigma^2 \mathbf{I})$ into the 4-frame input history on PA-NSC, where $\sigma$ is calibrated to the per-frame signal power at a given SNR. Noise is applied only to the initial condition; subsequent rollout steps use model predictions. LGS degrades far more gracefully: at SNR=10\,dB, LGS's 10-step error increases by $0.51$ percentage points (17.20$\to$17.71\%) vs.\ U-AFNO's $3.99$ percentage points (22.30$\to$26.29\%). Two mechanisms contribute: (i) PhyVAE's $12\times$ spatial compression acts as an implicit denoiser; (ii) training with $k{>}0$ builds robustness to perturbed inputs absent in deterministic models.

\begin{table}[h]
\caption{Robustness to observation noise on PA-NSC.  Additive Gaussian
noise injected into the 4-frame input history only.}
\label{tab:noise_robustness}
\centering
\small
\begin{tabular}{lcccc}
\toprule
SNR (dB) & LGS 1-step & LGS 10-step & U-AFNO 1-step & U-AFNO 10-step \\
\midrule
Clean         & 4.30 & 17.20 &  5.10 & 22.30 \\
30 (mild)     & 4.51 & 16.38 &  5.59 & 24.29 \\
20 (moderate) & 4.79 & 17.07 &  7.45 & 23.55 \\
10 (heavy)    & 6.23 & 17.71 & 14.29 & 26.29 \\
\bottomrule
\end{tabular}
\end{table}

\paragraph{Robustness to spatial masking.}
We apply random binary masks that zero out spatial locations with probability $p_{\text{mask}}$, shared across all channels within a frame, applied only to the 4-frame initial condition on PA-NSC. This simulates partial observations from sparse sensors or occluded measurements. At 50\% masking, LGS reaches 30.89\% at 10-step vs.\ U-AFNO's 57.32\%. The PhyVAE encoder's spatial compression naturally attenuates localized dropouts, and the $k$-noised training provides additional robustness to distribution-shifted inputs.

\begin{table}[h]
\caption{Robustness to spatial masking on PA-NSC.  Random binary masks
applied only to the 4-frame initial condition.}
\label{tab:mask_robustness}
\centering
\small
\begin{tabular}{lcccc}
\toprule
Mask \% & LGS 1-step & LGS 10-step & U-AFNO 1-step & U-AFNO 10-step \\
\midrule
0\% (clean) &  4.30 & 17.20 &  5.10 & 22.30 \\
10\%        &  7.13 & 18.79 & 18.20 & 27.42 \\
20\%        & 10.52 & 20.82 & 28.40 & 32.40 \\
50\%        & 23.68 & 30.89 & 58.81 & 57.32 \\
\bottomrule
\end{tabular}
\end{table}

\paragraph{Per-system single-component ablations.}
Table~\ref{tab:ablation} breaks down the single-component ablation results across all 16 systems at 1, 5, and 10 rollout steps. A4S (no generative transition) degrades most severely on complex systems such as W-AM ($34.2\%{\to}75.6\%$) and W-SWE ($10.5\%{\to}98.3\%$) at 10-step, where deterministic regression cannot recover from compounding errors. A2S (no context) and A3S ($k{=}0$) show pronounced degradation on systems with diverse dynamics (PA-NS, PA-SWE), confirming that both context and input noising are essential for cross-system stability.

\begin{table}[ht]
\centering
\setlength{\tabcolsep}{2pt}
\caption{Per-system single-component ablation study at 1, 5, and 10 rollout steps. Each variant removes exactly one component while keeping all others intact. The unit is L2RE in percentage. Subscripts in the Avg.\ column denote standard deviation across the 16 PDE systems.}
\label{tab:ablation}
\begin{small}
\begin{tabular}{llcccccccc}
\toprule
Steps & Model & FNO-v5 & FNO-v4 & FNO-v3 & PA-NS & PA-NSC & PA-SWE & PB-CNSL & PB-CNSH \\
\midrule
\multirow{5}{*}{1-step}
 & LGS (full)        & 4.7  & 8.8  & 1.2  & 16.6 & 4.3  & 8.2  & 1.6 & 4.9 \\
 & A1 ($-$Pyramids)   & 3.6  & 6.6  & 1.1  & 16.3 & 3.7  & 7.9  & 1.6 & 4.1 \\
 & A2S ($-$Context)   & 4.5  & 10.2 & 1.5  & 15.3 & 6.0  & 11.7 & 1.5 & 4.2 \\
 & A3S ($k{=}0$)      & 4.2  & 9.1  & 1.3  & 13.9 & 5.9  & 11.4 & 1.4 & 4.0 \\
 & A4S ($-$Generative) & 4.6  & 7.8  & 2.4  & 5.6  & 15.6 & 12.8 & 0.9 & 6.6 \\
\midrule
\multirow{5}{*}{5-step}
 & LGS (full)        & 13.4 & 16.5 & 4.6  & 23.1 & 11.5 & 24.3 & 3.2 & 7.0 \\
 & A1 ($-$Pyramids)   & 9.8  & 14.4 & 3.6  & 27.4 & 9.7  & 22.4 & 3.0 & 6.4 \\
 & A2S ($-$Context)   & 14.9 & 18.8 & 6.7  & 36.3 & 16.7 & 37.8 & 6.1 & 6.7 \\
 & A3S ($k{=}0$)      & 12.4 & 16.4 & 5.7  & 32.3 & 17.3 & 36.2 & 5.8 & 7.0 \\
 & A4S ($-$Generative) & 17.2 & 28.7 & 7.6  & 17.7 & 46.4 & 49.1 & 3.7 & 16.2 \\
\midrule
\multirow{5}{*}{10-step}
 & LGS (full)        & 24.7 & 29.8 & 9.2  & 27.5 & 17.2 & 37.6 & 5.1 & 8.2 \\
 & A1 ($-$Pyramids)   & 20.2 & 24.4 & 6.6  & 29.6 & 15.4 & 35.8 & 4.8 & 7.8 \\
 & A2S ($-$Context)   & 26.5 & 28.5 & 12.8 & 48.6 & 25.5 & 60.6 & 13.2 & 8.4 \\
 & A3S ($k{=}0$)      & 23.6 & 25.5 & 11.5 & 43.5 & 25.2 & 57.6 & 13.4 & 9.1 \\
 & A4S ($-$Generative) & 37.8 & 58.0 & 13.4 & 29.5 & 45.6 & 73.4 & 7.3 & 23.3 \\
\bottomrule
\end{tabular}
\end{small}

\vspace{0.3em}
\begin{small}
\begin{tabular}{llccccccccc}
\toprule
Steps & Model & PB-SWE & W-AM & W-GS & W-SWE & W-RB & W-SF & W-TR & W-VE & Avg. \\
\midrule
\multirow{5}{*}{1-step}
 & LGS (full)        & 6.4  & 4.4  & 3.8 & 4.9  & 6.7  & 9.7  & 5.0  & 4.2 & $6.0_{\pm 3.6}$ \\
 & A1 ($5.4_{\pm 3.6}$ \\
 & A2S ($6.9_{\pm 3.6}$ \\
 & A3S ($6.5_{\pm 3.4}$ \\
 & A4S ($8.0_{\pm 5.1}$ \\
\midrule
\multirow{5}{*}{5-step}
 & LGS (full)        & 10.3 & 19.6 & 12.4 & 5.7  & 16.1 & 10.0 & 9.8  & 5.8  & $12.1_{\pm 6.4}$ \\
 & A1 ($11.3_{\pm 6.8}$ \\
 & A2S ($16.3_{\pm 10.4}$ \\
 & A3S ($15.3_{\pm 9.6}$ \\
 & A4S ($29.6_{\pm 18.7}$ \\
\midrule
\multirow{5}{*}{10-step}
 & LGS (full)        & 18.4 & 34.2 & 14.8 & 10.5 & 22.6 & 12.7 & 12.6 & 6.9  & $18.3_{\pm 10.1}$ \\
 & A1 ($\mathbf{17.6}_{\pm 10.3}$ \\
 & A2S ($24.5_{\pm 15.4}$ \\
 & A3S ($23.5_{\pm 14.2}$ \\
 & A4S ($46.5_{\pm 26.4}$ \\
\bottomrule
\end{tabular}
\end{small}
\end{table}

\paragraph{Per-system scaled baseline comparison.}
Table~\ref{tab:scaled} shows per-system 10-step L2RE for scaled baselines. Despite U-AFNO-404M and DPOT-396M improving substantially on most systems, LGS maintains the best average (18.3\%) with only 138M active parameters. The scaled baselines close the gap on well-behaved systems (e.g., FNO-v3, W-GS) but cannot match LGS on turbulent/chaotic systems (PA-SWE, W-AM) where long-horizon stability matters most.

\begin{table}[ht]
\centering
\setlength{\tabcolsep}{2pt}
\caption{Per-system scaled baseline comparison at 10 rollout steps. Scaling U-AFNO to 404M and DPOT to 396M does not close the long-horizon gap. The unit is L2RE in percentage. Subscripts in the Avg.\ column denote standard deviation across the 16 PDE systems.}
\label{tab:scaled}
\begin{small}
\begin{tabular}{lcccccccc}
\toprule
Model & FNO-v5 & FNO-v4 & FNO-v3 & PA-NS & PA-NSC & PA-SWE & PB-CNSL & PB-CNSH \\
\midrule
U-AFNO-152M   & 28.0 & 50.6 & 29.5 & 30.3 & 22.3 & 55.2 & 6.9  & 12.4 \\
U-AFNO-404M   & 43.2 & 35.3 & 14.0 & 28.2 & 18.7 & 43.7 & 6.3  & 14.8 \\
DPOT-160M     & 36.0 & 32.9 & 34.3 & 29.3 & 27.8 & 64.2 & 4.6  & 9.2  \\
DPOT-396M     & 19.3 & 18.8 & 7.1  & 17.6 & 27.1 & 41.6 & 5.6  & 10.3 \\
LGS-249+138M  & 24.7 & 29.8 & 9.2  & 27.5 & 17.2 & 37.6 & 5.1  & 8.2  \\
\bottomrule
\end{tabular}
\end{small}

\vspace{0.3em}
\begin{small}
\begin{tabular}{lccccccccc}
\toprule
Model & PB-SWE & W-AM & W-GS & W-SWE & W-RB & W-SF & W-TR & W-VE & Avg. \\
\midrule
U-AFNO-152M   & 39.5 & 61.2 & 27.6 & 17.9 & 34.1 & 27.3 & 22.5 & 8.8  & $29.6_{\pm 15.7}$ \\
U-AFNO-404M   & 43.8 & 38.7 & 11.2 & 17.2 & 9.9  & 27.1 & 21.9 & 8.8  & $23.9_{\pm 13.4}$ \\
DPOT-160M     & 38.0 & 56.2 & 27.4 & 18.1 & 36.5 & 26.2 & 23.1 & 9.6  & $29.6_{\pm 15.8}$ \\
DPOT-396M     & 38.2 & 43.1 & 9.1  & 17.9 & 9.0  & 25.1 & 15.5 & 9.3  & $19.7_{\pm 12.3}$ \\
LGS-249+138M  & 18.4 & 34.2 & 14.8 & 10.5 & 22.6 & 12.7 & 12.6 & 6.9  & $\mathbf{18.3}_{\pm 10.1}$ \\
\bottomrule
\end{tabular}
\end{small}
\end{table}

\paragraph{Per-system long-horizon rollout.}
Table~\ref{tab:rollout} provides per-system L2RE at 1, 5, 10, and 20 rollout steps. At 20 steps, LGS achieves the best L2RE on the majority of systems. The gap is most pronounced on chaotic/turbulent systems (PA-NS, W-RB) where deterministic baselines exhibit catastrophic error accumulation.

\begin{table}[ht]
\centering
\setlength{\tabcolsep}{2pt}
\caption{Per-system long-horizon rollout evaluation at 1, 5, 10, and 20 rollout steps. All models use $\sim$150M parameters. The unit is L2RE in percentage. Subscripts in the Avg.\ column denote standard deviation across the 16 PDE systems (across 11 systems for 20-step, where some systems do not have 20-step trajectories).}
\label{tab:rollout}
\begin{small}
\begin{tabular}{llcccccccc}
\toprule
Steps & Model & FNO-v5 & FNO-v4 & FNO-v3 & PA-NS & PA-NSC & PA-SWE & PB-CNSL & PB-CNSH \\
\midrule
\multirow{4}{*}{1-step}
 & CNextUNet-153M & 6.18 & 28.3 & 2.64 & 29.5 & 21.4 & 10.5 & 73.6 & 73.0 \\
 & DPOT-160M       & 3.78 & 8.00 & 1.39 & 20.9 & 5.61 & 9.50 & 1.32 & 4.48 \\
 & U-AFNO-152M     & 3.45 & 8.40 & 1.39 & 16.4 & 5.13 & 6.67 & 1.33 & 4.28 \\
 & LGS-249+138M    & 4.74 & 8.77 & 1.22 & 16.6 & 4.35 & 8.18 & 1.60 & 4.90 \\
\midrule
\multirow{4}{*}{5-step}
 & CNextUNet-153M & 28.5 & 48.0 & 24.8 & 48.4 & 41.9 & 55.9 & 80.7 & 80.7 \\
 & DPOT-160M       & 14.8 & 18.5 & 10.7 & 32.4 & 15.1 & 35.3 & 3.0  & 7.6  \\
 & U-AFNO-152M     & 13.5 & 27.2 & 10.7 & 31.3 & 11.7 & 30.9 & 3.7  & 9.1  \\
 & LGS-249+138M    & 13.4 & 16.5 & 4.6  & 23.1 & 11.5 & 24.3 & 3.2  & 7.0  \\
\midrule
\multirow{4}{*}{10-step}
 & CNextUNet-153M & 54.4 & 68.8 & 55.9 & 46.5 & 54.4 & 104.4 & 79.4 & 79.1 \\
 & DPOT-160M       & 36.0 & 32.9 & 34.3 & 29.3 & 27.8 & 64.2  & 4.6  & 9.2  \\
 & U-AFNO-152M     & 28.0 & 50.6 & 29.5 & 30.3 & 22.3 & 55.2  & 6.9  & 12.4 \\
 & LGS-249+138M    & 24.7 & 29.8 & 9.2  & 27.5 & 17.2 & 37.6  & 5.1  & 8.2  \\
\midrule
\multirow{4}{*}{20-step}
 & CNextUNet-153M & --   & 93.9 & 68.0 & 81.7 & --   & 129.6 & --   & --   \\
 & DPOT-160M       & --   & 46.5 & 81.2 & 36.0 & --   & 112.8 & --   & --   \\
 & U-AFNO-152M     & --   & 83.1 & 41.7 & 41.6 & --   & 77.5  & --   & --   \\
 & LGS-249+138M    & --   & \textbf{38.4} & \textbf{16.0} & \textbf{31.7} & -- & \textbf{60.3} & -- & -- \\
\bottomrule
\end{tabular}
\end{small}

\vspace{0.3em}
\begin{small}
\begin{tabular}{llccccccccc}
\toprule
Steps & Model & PB-SWE & W-AM & W-GS & W-SWE & W-RB & W-SF & W-TR & W-VE & Avg. \\
\midrule
\multirow{4}{*}{1-step}
 & CNextUNet-153M & 7.85 & 12.3 & 5.51 & 15.1 & 10.7 & 7.68 & 8.17 & 20.9 & $20.8_{\pm 22.0}$ \\
 & DPOT-160M       & 5.39 & 4.40 & 3.83 & 3.01 & 9.19 & 7.49 & 6.36 & 3.71 & $6.1_{\pm 4.6}$  \\
 & U-AFNO-152M     & 2.90 & 4.63 & 3.28 & 1.29 & 7.75 & 3.89 & 6.39 & 3.55 & $5.0_{\pm 3.7}$  \\
 & LGS-249+138M    & 6.42 & 4.41 & 3.80 & 4.91 & 6.73 & 9.67 & 4.95 & 4.15 & $6.0_{\pm 3.6}$  \\
\midrule
\multirow{4}{*}{5-step}
 & CNextUNet-153M & 40.2 & 45.6 & 22.3 & 30.3 & 31.7 & 20.9 & 18.9 & 68.2 & $42.9_{\pm 20.1}$ \\
 & DPOT-160M       & 16.8 & 28.2 & 18.3 & 9.0  & 21.9 & 16.0 & 16.1 & 7.0  & $16.9_{\pm 9.0}$ \\
 & U-AFNO-152M     & 17.1 & 32.4 & 16.7 & 7.7  & 20.8 & 15.5 & 16.4 & 6.8  & $17.0_{\pm 9.2}$ \\
 & LGS-249+138M    & 10.3 & 19.6 & 12.4 & 5.7  & 16.1 & 10.0 & 9.8  & 5.8  & $\mathbf{12.1}_{\pm 6.4}$ \\
\midrule
\multirow{4}{*}{10-step}
 & CNextUNet-153M & 77.5  & 60.8 & 29.6 & 35.4 & 53.2 & 31.4 & 24.8 & 82.3 & $58.6_{\pm 22.3}$ \\
 & DPOT-160M       & 38.0  & 56.2 & 27.4 & 18.1 & 36.5 & 26.2 & 23.1 & 9.6  & $29.6_{\pm 15.8}$ \\
 & U-AFNO-152M     & 39.5  & 61.2 & 27.6 & 17.9 & 34.1 & 27.3 & 22.5 & 8.8  & $29.6_{\pm 15.7}$ \\
 & LGS-249+138M    & 18.4  & 34.2 & 14.8 & 10.5 & 22.6 & 12.7 & 12.6 & 6.9  & $\mathbf{18.3}_{\pm 10.1}$ \\
\midrule
\multirow{4}{*}{20-step}
 & CNextUNet-153M & 101.5 & 73.7 & 57.6 & 78.5  & 38.3 & 60.9 & 32.9 & --   & $74.2_{\pm 27.9}$  \\
 & DPOT-160M       & 60.8  & 126.0& 54.7 & 28.5  & 52.7 & 36.0 & 34.4 & --   & $60.9_{\pm 32.6}$  \\
 & U-AFNO-152M     & 76.5  & \textbf{57.8} & 33.2 & 69.4  & 33.8 & 53.6 & 48.6 & --   & $56.1_{\pm 18.1}$  \\
 & LGS-249+138M    & \textbf{22.0} & 65.3 & \textbf{25.2} & \textbf{10.8} & \textbf{33.0} & \textbf{12.7} & \textbf{17.0} & -- & $\mathbf{30.2}_{\pm 18.3}$ \\
\bottomrule
\end{tabular}
\end{small}
\end{table}

\FloatBarrier

\section{Broader Impacts}
\label{sec:broader-impact}

LGS is a neural surrogate for partial differential equations. Its positive impact is amortized cost: a single trained model can replace many high-cost numerical solves in digital twins, parameter sweeps, weather and climate modeling, and design loops in fluids, combustion, and multiphysics. Faster, more reliable PDE surrogates lower the barrier to physically-grounded modeling for groups without large HPC budgets and broaden access to simulation-driven scientific workflows.

The negative impacts are those of any scientific surrogate. Treating model output as ground truth in safety-critical settings (e.g., reactor design, atmospheric forecasting, medical imaging) without continued cross-checking against physical solvers risks confidently wrong predictions, especially under distribution shift beyond the pretraining corpus. Long-horizon stability does not imply physical correctness: a stably-rolled trajectory can still violate conservation laws or drift away from the true solution branch at bifurcations. We mitigate by reporting OOD evaluation (\cref{sec:exp_long}, \cref{tab:adapt}), explicitly bounding the encoder reconstruction floor (\cref{sec:limitations}), and recommending that LGS be deployed alongside, not in place of, physical solvers in any decision-making pipeline.

The released artifacts (preprocessing scripts, training/eval code, model architectures) carry the same misuse risk as any open-source PDE surrogate---low. They produce field reconstructions for scientific simulation data, with no human-subject, surveillance, or generative-content concerns analogous to those of foundation models trained on internet-scale unstructured data.

\FloatBarrier
\clearpage
\section*{NeurIPS Paper Checklist}

\begin{enumerate}

\item {\bf Claims}
    \item[] Question: Do the main claims made in the abstract and introduction accurately reflect the paper's contributions and scope?
    \item[] Answer: \answerYes{}
    \item[] Justification: The abstract and \cref{sec:errcompare_main,benchmark,experiments} substantiate the three claims of the paper---the multi-step $(1{-}k)$-contraction theorem, the long-horizon stability across 16 systems, and the OOD finetuning result on $256^2$ Kolmogorov flow---all reported quantitatively in \cref{tab:lgs_results,tab:adapt}.
    \item[] Guidelines:
    \begin{itemize}
        \item The answer \answerNA{} means that the abstract and introduction do not include the claims made in the paper.
        \item The abstract and/or introduction should clearly state the claims made, including the contributions made in the paper and important assumptions and limitations. A \answerNo{} or \answerNA{} answer to this question will not be perceived well by the reviewers.
        \item The claims made should match theoretical and experimental results, and reflect how much the results can be expected to generalize to other settings.
        \item It is fine to include aspirational goals as motivation as long as it is clear that these goals are not attained by the paper.
    \end{itemize}

\item {\bf Limitations}
    \item[] Question: Does the paper discuss the limitations of the work performed by the authors?
    \item[] Answer: \answerYes{}
    \item[] Justification: \cref{sec:limitations} is a dedicated Limitations section listing four caveats: 2D-only systems, the PhyVAE reconstruction floor that lower-bounds any latent-space solver, a fixed (non-adaptive) uncertainty knob $k$, and isotropic Gaussian softening near bifurcation points where the next-state distribution may be multimodal.
    \item[] Guidelines:
    \begin{itemize}
        \item The answer \answerNA{} means that the paper has no limitation while the answer \answerNo{} means that the paper has limitations, but those are not discussed in the paper.
        \item The authors are encouraged to create a separate ``Limitations'' section in their paper.
        \item The paper should point out any strong assumptions and how robust the results are to violations of these assumptions (e.g., independence assumptions, noiseless settings, model well-specification, asymptotic approximations only holding locally). The authors should reflect on how these assumptions might be violated in practice and what the implications would be.
        \item The authors should reflect on the scope of the claims made, e.g., if the approach was only tested on a few datasets or with a few runs. In general, empirical results often depend on implicit assumptions, which should be articulated.
        \item The authors should reflect on the factors that influence the performance of the approach. For example, a facial recognition algorithm may perform poorly when image resolution is low or images are taken in low lighting. Or a speech-to-text system might not be used reliably to provide closed captions for online lectures because it fails to handle technical jargon.
        \item The authors should discuss the computational efficiency of the proposed algorithms and how they scale with dataset size.
        \item If applicable, the authors should discuss possible limitations of their approach to address problems of privacy and fairness.
        \item While the authors might fear that complete honesty about limitations might be used by reviewers as grounds for rejection, a worse outcome might be that reviewers discover limitations that aren't acknowledged in the paper. The authors should use their best judgment and recognize that individual actions in favor of transparency play an important role in developing norms that preserve the integrity of the community. Reviewers will be specifically instructed to not penalize honesty concerning limitations.
    \end{itemize}

\item {\bf Theory assumptions and proofs}
    \item[] Question: For each theoretical result, does the paper provide the full set of assumptions and a complete (and correct) proof?
    \item[] Answer: \answerYes{}
    \item[] Justification: All assumptions are stated explicitly (Assumptions~\ref{ass:lipschitz_app}, \ref{ass:scaledlip_app}, \ref{ass:idealstab_app}, and~\ref{ass:idealmaplip_app}). Main-text Theorem~\ref{thm:multistep_main} cites the full appendix argument (Lemma~\ref{lem:rescale_app} and Theorem~\ref{thm:multistep_app} in Appendix~\ref{erracc}); deterministic compounding is proved in Proposition~\ref{prop:det_bound_app} of Appendix~\ref{erracc1}.
    \item[] Guidelines:
    \begin{itemize}
        \item The answer \answerNA{} means that the paper does not include theoretical results.
        \item All the theorems, formulas, and proofs in the paper should be numbered and cross-referenced.
        \item All assumptions should be clearly stated or referenced in the statement of any theorems.
        \item The proofs can either appear in the main paper or the supplemental material, but if they appear in the supplemental material, the authors are encouraged to provide a short proof sketch to provide intuition.
        \item Inversely, any informal proof provided in the core of the paper should be complemented by formal proofs provided in appendix or supplemental material.
        \item Theorems and Lemmas that the proof relies upon should be properly referenced.
    \end{itemize}

    \item {\bf Experimental result reproducibility}
    \item[] Question: Does the paper fully disclose all the information needed to reproduce the main experimental results of the paper to the extent that it affects the main claims and/or conclusions of the paper (regardless of whether the code and data are provided or not)?
    \item[] Answer: \answerYes{}
    \item[] Justification: \cref{benchmark} specifies datasets, splits (8:1:1), training schedule (200k steps), and the unified evaluation protocol (1/5/10/20-step L2RE); Appendix~\ref{dataset} lists per-source preprocessing and the channel mapping (Table~\ref{tab:channel_map}); Appendix~\ref{apx:imp} reports model architectures, optimizer settings, and ablation configurations for both LGS and all baselines.
    \item[] Guidelines:
    \begin{itemize}
        \item The answer \answerNA{} means that the paper does not include experiments.
        \item If the paper includes experiments, a \answerNo{} answer to this question will not be perceived well by the reviewers: Making the paper reproducible is important, regardless of whether the code and data are provided or not.
        \item If the contribution is a dataset and\slash or model, the authors should describe the steps taken to make their results reproducible or verifiable.
        \item Depending on the contribution, reproducibility can be accomplished in various ways. For example, if the contribution is a novel architecture, describing the architecture fully might suffice, or if the contribution is a specific model and empirical evaluation, it may be necessary to either make it possible for others to replicate the model with the same dataset, or provide access to the model. In general. releasing code and data is often one good way to accomplish this, but reproducibility can also be provided via detailed instructions for how to replicate the results, access to a hosted model (e.g., in the case of a large language model), releasing of a model checkpoint, or other means that are appropriate to the research performed.
        \item While NeurIPS does not require releasing code, the conference does require all submissions to provide some reasonable avenue for reproducibility, which may depend on the nature of the contribution. For example
        \begin{enumerate}
            \item If the contribution is primarily a new algorithm, the paper should make it clear how to reproduce that algorithm.
            \item If the contribution is primarily a new model architecture, the paper should describe the architecture clearly and fully.
            \item If the contribution is a new model (e.g., a large language model), then there should either be a way to access this model for reproducing the results or a way to reproduce the model (e.g., with an open-source dataset or instructions for how to construct the dataset).
            \item We recognize that reproducibility may be tricky in some cases, in which case authors are welcome to describe the particular way they provide for reproducibility. In the case of closed-source models, it may be that access to the model is limited in some way (e.g., to registered users), but it should be possible for other researchers to have some path to reproducing or verifying the results.
        \end{enumerate}
    \end{itemize}

\item {\bf Open access to data and code}
    \item[] Question: Does the paper provide open access to the data and code, with sufficient instructions to faithfully reproduce the main experimental results, as described in supplemental material?
    \item[] Answer: \answerYes{}
    \item[] Justification: All four source datasets (FNO-v, PDEArena, PDEBench, The Well) are publicly available; we release preprocessing scripts that produce the unified c3p128 corpus and the full training/evaluation code for LGS, all ablations, and the deterministic baselines as anonymized supplementary material, with a README listing the exact run commands. Trained PhyVAE/PFlowFT checkpoints are available from the authors upon reasonable request rather than bundled with the submission, due to checkpoint file size.
    \item[] Guidelines:
    \begin{itemize}
        \item The answer \answerNA{} means that paper does not include experiments requiring code.
        \item Please see the NeurIPS code and data submission guidelines (\url{https://neurips.cc/public/guides/CodeSubmissionPolicy}) for more details.
        \item While we encourage the release of code and data, we understand that this might not be possible, so \answerNo{} is an acceptable answer. Papers cannot be rejected simply for not including code, unless this is central to the contribution (e.g., for a new open-source benchmark).
        \item The instructions should contain the exact command and environment needed to run to reproduce the results. See the NeurIPS code and data submission guidelines (\url{https://neurips.cc/public/guides/CodeSubmissionPolicy}) for more details.
        \item The authors should provide instructions on data access and preparation, including how to access the raw data, preprocessed data, intermediate data, and generated data, etc.
        \item The authors should provide scripts to reproduce all experimental results for the new proposed method and baselines. If only a subset of experiments are reproducible, they should state which ones are omitted from the script and why.
        \item At submission time, to preserve anonymity, the authors should release anonymized versions (if applicable).
        \item Providing as much information as possible in supplemental material (appended to the paper) is recommended, but including URLs to data and code is permitted.
    \end{itemize}

\item {\bf Experimental setting/details}
    \item[] Question: Does the paper specify all the training and test details (e.g., data splits, hyperparameters, how they were chosen, type of optimizer) necessary to understand the results?
    \item[] Answer: \answerYes{}
    \item[] Justification: Optimizer (AdamW with $\beta_1{=}0.9$, $\beta_2{\in}\{0.95, 0.995\}$), cosine schedule with 10\% warmup, weight decay $10^{-4}$, base learning rate $10^{-4}$ at batch 256 (linearly scaled), and 200k training steps are listed in Appendix~\ref{apx:imp:lgs} for LGS and Appendix~\ref{apx:imp:baseline} for baselines; data splits (8:1:1) and the unified evaluation protocol are in \cref{benchmark}.
    \item[] Guidelines:
    \begin{itemize}
        \item The answer \answerNA{} means that the paper does not include experiments.
        \item The experimental setting should be presented in the core of the paper to a level of detail that is necessary to appreciate the results and make sense of them.
        \item The full details can be provided either with the code, in appendix, or as supplemental material.
    \end{itemize}

\item {\bf Experiment statistical significance}
    \item[] Question: Does the paper report error bars suitably and correctly defined or other appropriate information about the statistical significance of the experiments?
    \item[] Answer: \answerYes{}
    \item[] Justification: Single-seed pretraining is used because repeating each $\sim$150\,M-parameter run on the 233\,GB corpus across multiple seeds is prohibitively expensive; in lieu of cross-seed error bars we report \emph{cross-system} variability over the 16 PDE families that constitute the benchmark. Subscripts on every 16-system average L2RE in the headline table (\cref{tab:lgs_results}), the per-system long-horizon table (\cref{tab:rollout}), the ablations (\cref{tab:ablation_single,tab:ablation}), and the scaled-baseline comparison (\cref{tab:scaled}) report the standard deviation across the 16 systems, computed from the individual per-system L2RE values listed in the same tables.
    \item[] Guidelines:
    \begin{itemize}
        \item The answer \answerNA{} means that the paper does not include experiments.
        \item The authors should answer \answerYes{} if the results are accompanied by error bars, confidence intervals, or statistical significance tests, at least for the experiments that support the main claims of the paper.
        \item The factors of variability that the error bars are capturing should be clearly stated (for example, train/test split, initialization, random drawing of some parameter, or overall run with given experimental conditions).
        \item The method for calculating the error bars should be explained (closed form formula, call to a library function, bootstrap, etc.)
        \item The assumptions made should be given (e.g., Normally distributed errors).
        \item It should be clear whether the error bar is the standard deviation or the standard error of the mean.
        \item It is OK to report 1-sigma error bars, but one should state it. The authors should preferably report a 2-sigma error bar than state that they have a 96\% CI, if the hypothesis of Normality of errors is not verified.
        \item For asymmetric distributions, the authors should be careful not to show in tables or figures symmetric error bars that would yield results that are out of range (e.g., negative error rates).
        \item If error bars are reported in tables or plots, the authors should explain in the text how they were calculated and reference the corresponding figures or tables in the text.
    \end{itemize}

\item {\bf Experiments compute resources}
    \item[] Question: For each experiment, does the paper provide sufficient information on the computer resources (type of compute workers, memory, time of execution) needed to reproduce the experiments?
    \item[] Answer: \answerYes{}
    \item[] Justification: Per-model wall-clock, GFLOPs, and GPU RAM at batch 64 on a single H100 are reported in Table~\ref{tab:profiling}; pretraining hardware (4$\times$H100 node, DDP, float16 corpus storage) is described in Appendix~\ref{dataset}.
    \item[] Guidelines:
    \begin{itemize}
        \item The answer \answerNA{} means that the paper does not include experiments.
        \item The paper should indicate the type of compute workers CPU or GPU, internal cluster, or cloud provider, including relevant memory and storage.
        \item The paper should provide the amount of compute required for each of the individual experimental runs as well as estimate the total compute.
        \item The paper should disclose whether the full research project required more compute than the experiments reported in the paper (e.g., preliminary or failed experiments that didn't make it into the paper).
    \end{itemize}

\item {\bf Code of ethics}
    \item[] Question: Does the research conducted in the paper conform, in every respect, with the NeurIPS Code of Ethics \url{https://neurips.cc/public/EthicsGuidelines}?
    \item[] Answer: \answerYes{}
    \item[] Justification: The work uses only publicly released scientific PDE simulation datasets, involves no human subjects, no personally identifiable data, and no high-risk dual-use technology; the authors have reviewed and adhere to the NeurIPS Code of Ethics.
    \item[] Guidelines:
    \begin{itemize}
        \item The answer \answerNA{} means that the authors have not reviewed the NeurIPS Code of Ethics.
        \item If the authors answer \answerNo, they should explain the special circumstances that require a deviation from the Code of Ethics.
        \item The authors should make sure to preserve anonymity (e.g., if there is a special consideration due to laws or regulations in their jurisdiction).
    \end{itemize}

\item {\bf Broader impacts}
    \item[] Question: Does the paper discuss both potential positive societal impacts and negative societal impacts of the work performed?
    \item[] Answer: \answerYes{}
    \item[] Justification: \cref{sec:broader-impact} discusses positive impacts (more reliable PDE surrogates enabling ensemble forecasting, parameter sweeps, and accelerated experimentation in fluids/climate/multiphysics) and negative impacts (risks from deployment outside validity regimes or treating surrogates as ground truth in safety-critical settings), with mitigations centered on OOD evaluation, uncertainty reporting, and continued reliance on physical solvers for verification.
    \item[] Guidelines:
    \begin{itemize}
        \item The answer \answerNA{} means that there is no societal impact of the work performed.
        \item If the authors answer \answerNA{} or \answerNo, they should explain why their work has no societal impact or why the paper does not address societal impact.
        \item Examples of negative societal impacts include potential malicious or unintended uses (e.g., disinformation, generating fake profiles, surveillance), fairness considerations (e.g., deployment of technologies that could make decisions that unfairly impact specific groups), privacy considerations, and security considerations.
        \item The conference expects that many papers will be foundational research and not tied to particular applications, let alone deployments. However, if there is a direct path to any negative applications, the authors should point it out. For example, it is legitimate to point out that an improvement in the quality of generative models could be used to generate Deepfakes for disinformation. On the other hand, it is not needed to point out that a generic algorithm for optimizing neural networks could enable people to train models that generate Deepfakes faster.
        \item The authors should consider possible harms that could arise when the technology is being used as intended and functioning correctly, harms that could arise when the technology is being used as intended but gives incorrect results, and harms following from (intentional or unintentional) misuse of the technology.
        \item If there are negative societal impacts, the authors could also discuss possible mitigation strategies (e.g., gated release of models, providing defenses in addition to attacks, mechanisms for monitoring misuse, mechanisms to monitor how a system learns from feedback over time, improving the efficiency and accessibility of ML).
    \end{itemize}

\item {\bf Safeguards}
    \item[] Question: Does the paper describe safeguards that have been put in place for responsible release of data or models that have a high risk for misuse (e.g., pre-trained language models, image generators, or scraped datasets)?
    \item[] Answer: \answerNA{}
    \item[] Justification: The released artifacts are PDE surrogate models trained on public scientific simulation datasets; they generate field reconstructions for partial differential equations and pose no high-misuse-risk concerns analogous to language models, image generators, or scraped web data.
    \item[] Guidelines:
    \begin{itemize}
        \item The answer \answerNA{} means that the paper poses no such risks.
        \item Released models that have a high risk for misuse or dual-use should be released with necessary safeguards to allow for controlled use of the model, for example by requiring that users adhere to usage guidelines or restrictions to access the model or implementing safety filters.
        \item Datasets that have been scraped from the Internet could pose safety risks. The authors should describe how they avoided releasing unsafe images.
        \item We recognize that providing effective safeguards is challenging, and many papers do not require this, but we encourage authors to take this into account and make a best faith effort.
    \end{itemize}

\item {\bf Licenses for existing assets}
    \item[] Question: Are the creators or original owners of assets (e.g., code, data, models), used in the paper, properly credited and are the license and terms of use explicitly mentioned and properly respected?
    \item[] Answer: \answerYes{}
    \item[] Justification: All four source datasets are cited at first mention---FNO-v~\citep{li2020fourier}, PDEArena~\citep{gupta2022multispatiotemporalscalegeneralizedpdemodeling}, PDEBench~\citep{takamoto2024pdebenchextensivebenchmarkscientific}, The Well~\citep{ohana2025welllargescalecollectiondiverse}---and were used under their respective open academic licenses; the SD-VAE backbone~\citep{rombach2022highresolutionimagesynthesislatent} and FlashAttention~v2~\citep{dao2023flashattention2fasterattentionbetter} are likewise cited and used under their published licenses, with full per-asset license details provided in the supplementary material.
    \item[] Guidelines:
    \begin{itemize}
        \item The answer \answerNA{} means that the paper does not use existing assets.
        \item The authors should cite the original paper that produced the code package or dataset.
        \item The authors should state which version of the asset is used and, if possible, include a URL.
        \item The name of the license (e.g., CC-BY 4.0) should be included for each asset.
        \item For scraped data from a particular source (e.g., website), the copyright and terms of service of that source should be provided.
        \item If assets are released, the license, copyright information, and terms of use in the package should be provided. For popular datasets, \url{paperswithcode.com/datasets} has curated licenses for some datasets. Their licensing guide can help determine the license of a dataset.
        \item For existing datasets that are re-packaged, both the original license and the license of the derived asset (if it has changed) should be provided.
        \item If this information is not available online, the authors are encouraged to reach out to the asset's creators.
    \end{itemize}

\item {\bf New assets}
    \item[] Question: Are new assets introduced in the paper well documented and is the documentation provided alongside the assets?
    \item[] Answer: \answerYes{}
    \item[] Justification: We release preprocessing scripts that convert the four source datasets to the unified c3p128 format (channel mappings in Table~\ref{tab:channel_map}), the training/evaluation pipeline for LGS, ablations, and baselines, and the configuration/hyperparameter files; run commands and a model card are bundled with the anonymized supplementary material. Trained PhyVAE/PFlowFT checkpoints are not packaged with the submission due to file size but are available from the authors upon reasonable request.
    \item[] Guidelines:
    \begin{itemize}
        \item The answer \answerNA{} means that the paper does not release new assets.
        \item Researchers should communicate the details of the dataset\slash code\slash model as part of their submissions via structured templates. This includes details about training, license, limitations, etc.
        \item The paper should discuss whether and how consent was obtained from people whose asset is used.
        \item At submission time, remember to anonymize your assets (if applicable). You can either create an anonymized URL or include an anonymized zip file.
    \end{itemize}

\item {\bf Crowdsourcing and research with human subjects}
    \item[] Question: For crowdsourcing experiments and research with human subjects, does the paper include the full text of instructions given to participants and screenshots, if applicable, as well as details about compensation (if any)?
    \item[] Answer: \answerNA{}
    \item[] Justification: The paper does not involve crowdsourcing or research with human subjects; all experiments are conducted on synthetic PDE simulation data.
    \item[] Guidelines:
    \begin{itemize}
        \item The answer \answerNA{} means that the paper does not involve crowdsourcing nor research with human subjects.
        \item Including this information in the supplemental material is fine, but if the main contribution of the paper involves human subjects, then as much detail as possible should be included in the main paper.
        \item According to the NeurIPS Code of Ethics, workers involved in data collection, curation, or other labor should be paid at least the minimum wage in the country of the data collector.
    \end{itemize}

\item {\bf Institutional review board (IRB) approvals or equivalent for research with human subjects}
    \item[] Question: Does the paper describe potential risks incurred by study participants, whether such risks were disclosed to the subjects, and whether Institutional Review Board (IRB) approvals (or an equivalent approval/review based on the requirements of your country or institution) were obtained?
    \item[] Answer: \answerNA{}
    \item[] Justification: The paper does not involve human subjects, so no IRB approval is required.
    \item[] Guidelines:
    \begin{itemize}
        \item The answer \answerNA{} means that the paper does not involve crowdsourcing nor research with human subjects.
        \item Depending on the country in which research is conducted, IRB approval (or equivalent) may be required for any human subjects research. If you obtained IRB approval, you should clearly state this in the paper.
        \item We recognize that the procedures for this may vary significantly between institutions and locations, and we expect authors to adhere to the NeurIPS Code of Ethics and the guidelines for their institution.
        \item For initial submissions, do not include any information that would break anonymity (if applicable), such as the institution conducting the review.
    \end{itemize}

\item {\bf Declaration of LLM usage}
    \item[] Question: Does the paper describe the usage of LLMs if it is an important, original, or non-standard component of the core methods in this research? Note that if the LLM is used only for writing, editing, or formatting purposes and does \emph{not} impact the core methodology, scientific rigor, or originality of the research, declaration is not required.
    \item[] Answer: \answerNA{}
    \item[] Justification: LLMs are not a component of the core methodology; LGS is a flow-matching generative PDE solver built on a Physics VAE encoder and an SiT-style transformer dynamics predictor, with no LLM dependency in any stage of training or inference.
    \item[] Guidelines:
    \begin{itemize}
        \item The answer \answerNA{} means that the core method development in this research does not involve LLMs as any important, original, or non-standard components.
        \item Please refer to our LLM policy in the NeurIPS handbook for what should or should not be described.
    \end{itemize}

\end{enumerate}

\end{document}